\documentclass{article}

\usepackage{microtype}
\usepackage{graphicx}
\usepackage{booktabs} %
\usepackage{arxiv}
\usepackage[numbers]{natbib}

\usepackage{hyperref}
\usepackage[utf8]{inputenc} %
\usepackage[T1]{fontenc}    %
\usepackage{hyperref}       %
\usepackage{url}            %
\usepackage{booktabs}       %
\usepackage{amsfonts}       %
\usepackage{nicefrac}       %
\usepackage{microtype}      %
\usepackage{makecell}
\usepackage{amsmath}
\usepackage[export]{adjustbox}
\usepackage{etoolbox}

\usepackage{amsthm}
\usepackage{graphicx}
\usepackage{xcolor}
\usepackage{comment}
\usepackage{subcaption}
\usepackage[font=small,labelfont=bf]{caption}
\usepackage{multirow}
\usepackage{overpic}
\usepackage{bbm}
\usepackage{wrapfig}
\usepackage{ulem}

\usepackage[ruled,vlined]{algorithm2e} 
\usepackage{soul}

\usepackage{siunitx}

\usepackage{amsmath,amsfonts,bm}

\newcommand{\eqd}{\,{ =_d} \,}

\def\eqref#1{equation~\ref{#1}}

\def\1{\bm{1}}

\def\ra{{\textnormal{a}}}

\def\ru{{\textnormal{u}}}

\def\rx{{\textnormal{x}}}
\def\ry{{\textnormal{y}}}
\def\rz{{\textnormal{z}}}

\def\rva{{\mathbf{a}}}

\def\rvu{{\mathbf{i}}}

\def\rvu{{\mathbf{u}}}

\def\rvx{{\mathbf{x}}}
\def\rvy{{\mathbf{y}}}
\def\rvz{{\mathbf{z}}}

\def\vtheta{{\bm{\theta}}}
\def\va{{\bm{a}}}

\def\vu{{\bm{u}}}
\def\vv{{\bm{v}}}
\def\vw{{\bm{w}}}
\def\vx{{\bm{x}}}
\def\vy{{\bm{y}}}
\def\vz{{\bm{z}}}

\def\mI{{\bm{I}}}

\def\mV{{\bm{V}}}

\DeclareMathAlphabet{\mathsfit}{\encodingdefault}{\sfdefault}{m}{sl}
\SetMathAlphabet{\mathsfit}{bold}{\encodingdefault}{\sfdefault}{bx}{n}

\def\sA{{\mathbb{A}}}

\def\sD{{\mathbb{D}}}

\def\sX{{\mathbb{X}}}
\def\sY{{\mathbb{Y}}}
\def\sZ{{\mathbb{Z}}}

\newcommand{\E}{\mathbb{E}}

\usepackage{mathtools}
\usepackage{array}
\newcolumntype{L}{>{\centering\arraybackslash}m{3cm}}

\usepackage{makecell}

\usepackage{siunitx}
\makeatletter
\newcommand{\printfnsymbol}[1]{%
  \textsuperscript{\@fnsymbol{#1}}%
}
\makeatother

\newcommand{\gancond}{GAN conditioning}
\newcommand{\ugen}{unconditional generator}
\newcommand{\cgen}{conditional generator}
\newcommand{\irep}{input reprogramming}

\newcommand{\gan}{GAN}
\newcommand{\cg}{CGAN}

\newcommand{\rg}{\textsc{InRep+}}
\newcommand{\ftg}{{Fine-tuning}}
\newcommand{\grep}{\textsc{GAN-Reprogram}}

\newcommand{\ag}{{ACGAN}}
\newcommand{\pg}{{ProjGAN}}
\newcommand{\ctg}{{ContraGAN}}

\newcommand{\dm}{MNIST}
\newcommand{\df}{{Fashion-MNIST}}
\newcommand{\dc}{CIFAR10}
\newcommand{\dhc}{CIFAR100}

\newcommand{\dg}{Gaussian mixture}
\newcommand{\dimg}{ImageNet}

\newcommand{\cas}{CAS}
\newcommand{\fid}{FID}
\newcommand{\cfid}{Intra-FID}

\newcommand{\real}{\textrm{real}}
\newcommand{\fake}{\textrm{fake}}

\DeclareMathOperator{\Bernoulli}{Bern}
\DeclareMathOperator{\Unif}{Unif}
\DeclareMathOperator{\sgn}{sgn}

\newtheorem{theorem}{Theorem}
\newtheorem{proposition}{Proposition}

\newtheorem{lemma}{Lemma}
\newtheorem{remark}{Remark}

\newcommand\td[1]{#1}

\newcommand\std[1]{\tiny$\pm$#1}

\makeatletter

\let\svthefootnote\thefootnote
\newcommand\freefootnote[1]{%
  \let\thefootnote\relax%
  \footnotetext{#1}%
  \let\thefootnote\svthefootnote%
}

\newcommand*\dif{\mathop{}\!\mathrm{d}}

\makeatother

\title{Improved Input Reprogramming for GAN Conditioning}

\author{ 
Tuan Dinh\footnotemark[2]
,\ \ Daewon Seo\footnotemark[4] 
,\ \  Zhixu Du\footnotemark[5] 
,\ \ Liang Shang\footnotemark[2] 
,\ \ Kangwook Lee\footnotemark[2] \\ \\
\normalsize \footnotemark[2] \ \ University of Wisconsin-Madison, USA\\
\footnotemark[4] \ \ Daegu Gyeongbuk Institute of Science and Technology, South Korea \\
\footnotemark[5] \ \ University of Hong Kong, Hong Kong
}

\begin{document}

\maketitle
\normalem

\freefootnote{Email: Tuan Dinh (tuan.dinh@wisc.edu)}

\begin{abstract}
We study the \gancond{} problem, whose goal is to convert a pretrained unconditional GAN into a conditional GAN using labeled data.
We first identify and analyze three approaches to this problem -- conditional GAN training from scratch, fine-tuning, and input reprogramming.
Our analysis reveals that when the amount of labeled data is small, input reprogramming performs the best.
Motivated by real-world scenarios with scarce labeled data, we focus on the input reprogramming approach and carefully analyze the existing algorithm. 
After identifying a few critical issues of the previous input reprogramming approach, we propose a new algorithm called \rg{}.
Our algorithm \rg{} addresses the existing issues with the novel uses of invertible neural networks and Positive-Unlabeled (PU) learning.
Via extensive experiments, we show that \rg{} outperforms all existing methods, particularly when label information is scarce, noisy, and/or imbalanced. 
For instance, for the task of conditioning a \dc{} GAN with $1\%$ labeled data, \rg{} achieves an average \cfid{} of $76.24$, whereas the second-best method achieves $114.51$. 
\end{abstract}

\section{Introduction}  \label{sec:intro}

Generative Adversarial Networks (\gan{}s)~\citep{goodfellow2014generative} have introduced an effective paradigm for modeling complex high dimensional distributions, such as natural images~\citep{miyato2018cgans,lucic2019high,kang2020contragan,isola2017image,zhu2017unpaired,yu2017seqgan,guo2018long},
videos~\citep{dong2019fw,kim2020jsi}, audios~\citep{yu2017seqgan, kong2020hifi} and texts~\citep{yu2017seqgan,guo2018long,brown2020language,blodgett2020language}.
With recent advancements in the design of well-behaved objectives~\citep{martin2017wasserstein}, regularization techniques~\citep{miyato2018spectral}, and scalable training for large models~\citep{biggan_training}, \gan{}s have achieved impressively realistic data generation.

Conditioning has become an essential research topic of GANs.
While earlier works focus on unconditional GANs (UGANs), which sample data from unconditional data distributions, conditional GANs (CGANs) have recently gained a more significant deal of attention thanks to their ability to generate high-quality samples from class-conditional data distributions~\citep{miyato2018cgans,mirza2014conditional,odena2017conditional}.
CGANs provide a broader range of applications in conditional image generation~\citep{mirza2014conditional}, text-to-image generation~\citep{reed2016generative}, image-to-image translation~\citep{isola2017image}, and text-to-speech synthesis~\citep{binkowski2019high,yamamoto2020parallel}.

In this work, we define a new problem, which we dub \emph{\gancond{}}, whose goal is to learn a CGAN given (a) a pretrained UGAN and (b) labeled data. 
The pretrained UGAN is given in the form of an \ugen{}. 
Also, we assume that the classes of the labeled data are exclusive to each other, \td{that is}, the true class-conditional distributions are separable.
Fig.~\ref{fig:problem} illustrates the problem setting with a two-class \dm{} dataset.
The first input, shown on the top left of the figure, is an \ugen{} $G$, which is trained on a mixed dataset of classes $0$ and $1$. 
The second input, shown on the bottom left of the figure, is a labeled dataset.
In this example, the goal of \gancond{} algorithms is to learn a two-class conditional generator $G'$ from these two inputs, as shown on the right of Fig.~\ref{fig:problem}.

Formally, given an unconditional generator $G$ trained with unlabeled data drawn from $p_{\text{data}}(\vx)$ and a conditional dataset $\sD$ where $\sD = \bigcup_{y\in \sY}\sD(y)$ with  $\sY$ being the label set and $\sD(y)$ being drawn from $p_{\text{data}}(\vx|y)$, \gancond{} algorithms learn a conditional generator $G'$ that generates $y$-conditional samples given label $y \in \mathbb{Y}$.
The formulation of the \gancond{} problem is motivated by several practical scenarios.
The first scenario is \textbf{pipelined training}.
For illustration, we consider a learner who wants to learn a text-to-speech algorithm using CGANs.
While labeled data is being collected (\td{that is}, text-speech pairs), the learner can start making use of a large amount of unlabeled data, which is publicly available, by pretraining a UGAN.
Once labeled data is collected, the learner can then use both the pretrained UGAN and the labeled data to learn a CGAN more efficiently.
This approach enables a pipelined training process to utilize better the long waiting time required for labeling.
Secondly, \gancond{} is helpful in a specific \textbf{online or streaming learning} setting, where the first part of the data stream is unlabeled, and the remaining data stream is labeled. 
Assuming that it is impossible to store the streamed data due to storage constraints, one must learn something in an online fashion when samples are available and then discard them right away. 
Now consider the setting where the final goal is to train a CGAN.  
One may discard the unlabeled part of the data stream and train a CGAN on the labeled part, but this will be strictly suboptimal. 
One plausible approach to this problem is to train a UGAN on the unlabeled part of the stream and then apply \gancond{} with the labeled part of the stream.
The third scenario is in \textbf{transferring knowledge from models pretrained on private data.} 
In many cases, pretrained UGAN models are publicly available while the training data is not, primarily because of privacy. 
An efficient algorithm for \gancond{} can be used to transfer knowledge from such pretrained UGANs when training a CGAN.
 
Existing approaches to \gancond{} can be categorized into three classes.
The first and most straightforward approach is discarding UGAN and applying the \cg{} training algorithms~\citep{miyato2018cgans,mirza2014conditional,odena2017conditional} on the labeled data.
However, training a \cg{} from scratch not only faces performance degradation if the labeled data is scarce or noisy~\citep{kodali2017convergence,odena2019open,shahbazi2022collapse} but also incurs enormous resources in terms of time, computation and memory.
The second approach is fine-tuning the UGAN into a CGAN.
While fine-tuning GANs~\citep{wangl2018transfer,wang2020mineGAN} provides a more efficient solution than the full \cg{} training, this method may suffer from the catastrophic forgetting phenomenon~\citep{li2017learning}.

Recent studies~\citep{engel2018latent,lee2020} propose a new approach to \gancond{}, called \emph{\irep{}}.
They show that this approach can achieve promising performances with remarkable computing savings.
However, their frameworks~\citep{engel2018latent,lee2020} are designed to handle only one-class datasets.
Therefore, to handle multi-class datasets, one must repeatedly apply the algorithm to each class, incurring huge memory when the number of classes is large.
Furthermore, the full CGAN training methods still achieve better conditioning performances than the existing \irep{} methods when the labeled data is sufficiently large.
Also, it remains unclear how the performance of input reprogramming-based approaches compares with that of the other approaches as the quality and amount of labeled data vary.

\begin{figure}
    \centering
	\includegraphics[width=0.9\linewidth]{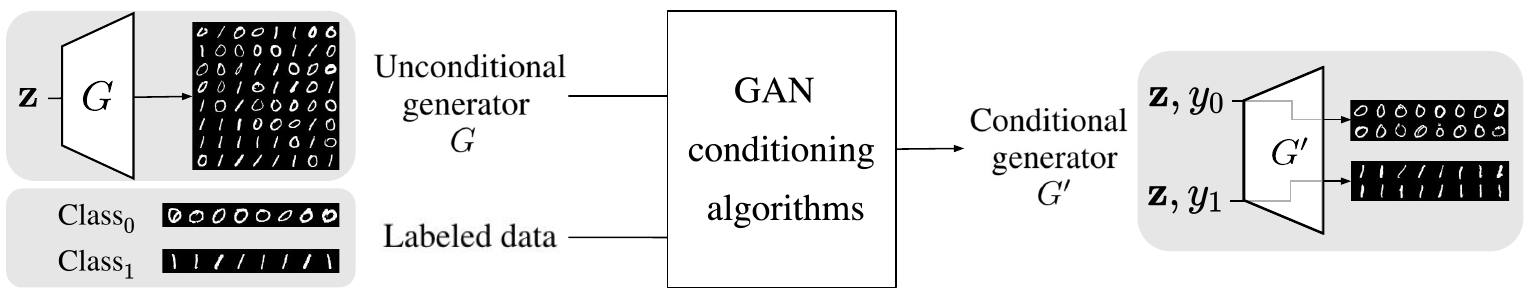}
	\vspace{-0.3em}
	\caption{\textbf{\gancond{} setting (illustrated with two-class \dm{}).}
	\gancond{} algorithms convert an \ugen{} $G$ (top left) into a \cgen{} $G'$ (right) using a labeled dataset (bottom left). 
	On the two-class \dm{} data, \ugen{} $G$ uniformly generates images of $0$s or $1$s from random noise vector $\rvz$. 
	The labeled data contains class-conditional images with labels $0$ and $1$.
	The output of the \gancond{} algorithm is the \cgen{} $G'$ that generates samples of class $y$ from random noise $\rvz$ and the provided class label $y$ in $\{y_0=0, y_1=1\}$.
	}
	\label{fig:problem}
	\vspace{-0.9em}
\end{figure}

In this work, we thoroughly study the possibility of the \irep{} framework for \gancond{}.
We analyze the limitations of the existing algorithms and propose \rg{} as an improved \irep{} framework to fully address the problems identified.
Shown in Fig.~\ref{fig:model} is the design of our framework.
\rg{} learns a conditional network $M$, called modifier, that transforms a random noise $\vz$ into a $y$-conditional noise $\vz_y$ from which the unconditional generator $G$ generates a $y$-conditional sample.
We learn the modifier network via adversarial training.
\rg{} adopts the invertible architecture~\citep{jacobsen2018irevnet,behrmann2019invertible,song2019mintnet} for the modifier to prevent the class-overlapping in the latent space.
We address the large memory issue by sharing the learnable networks between classes.
Also, we make use of Positive-Unlabeled learning (PU-learning)~\citep{guo2020positive} for the discriminator loss to overcome the training instability.

\begin{figure}
	   \hspace*{.13in}
    	\includegraphics[width=1\linewidth]{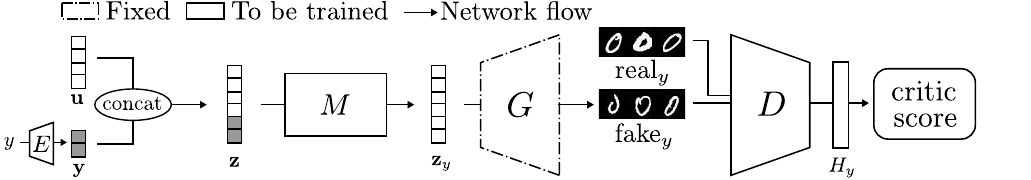}
	\caption{\textbf{Modular design of \rg{} (Improved Input Reprogramming) framework.}
	Given a fixed \ugen{} $G$, \rg{} learns a modifier network $M$ and conditional discriminator networks $D_{y}$, $y \in \sY$.
	Each $D_{y}$ is built on a weight-sharing network $D$ followed by a class-specific linear head $H_{y}$, i.e. $D_{y} = H_{y}\circ D$.
	We embed each label $y$ into a vector $\rvy$ with the embedding module $E$, then concatenate $\rvy$ to a random noise vector $\rvu$ to get vector $\rvz$.
	The modifier $M$ converts $\rvz$ into a $y$-conditional noise $\rvz_y$ so that $G(\rvz_y)$ is a sample of class $y$ (a $0$-image of $\text{fake}_{y}$ in the illustration).
	We train our networks using GAN training.
	} 
	\label{fig:model}
	\vspace{-2mm}
\end{figure}
Via theoretical analysis, we show that \rg{} is optimal with the guaranteed convergence given the optimal UGAN.
Our extensive empirical study shows that \rg{} can efficiently learn high-quality samples that are correctly conditioned, 
achieving state-of-the-art performances regarding various quantitative measures.
In particular, \rg{} significantly outperforms other approaches on various datasets when the amount of labeled data is as small as $1$\% or $10$\% of the amount of unlabeled data used for training UGAN.
We also demonstrate the robustness of \rg{} against label-noisy and class-imbalanced labeled data.

The rest of our paper is organized as follows. 
We first review the existing approaches to \gancond{} in Sec.~\ref{sec:prelim}.
In Sec.~\ref{sec:inrep}, we analyze the existing \irep{} framework for \gancond{} and propose our new algorithm \rg{}. 
In Sec.~\ref{sec:experiment}, we empirically evaluate \rg{} and other \gancond{} methods in various training settings.
Finally, Sec.~\ref{sec:discussion} provides further discussions on the limitation of our proposed approach as well as the applicability of \rg{} to other generative models and prompt tuning.
We conclude our paper in Sec.~\ref{sec:conclusion}.

\paragraph{Notation} We follow the notations and symbols in the standard \gan{} literature~\citep{goodfellow2014generative,goodfellow2016deep}. 
That is, $a, \va, \ra, \rva, \sA$ denote a scalar, vector, random scalar variable, random vector variable, and set; $\vx, y, \vz$ denote a feature vector, label, and input noise vector, respectively.
\section{Preliminaries on \gancond{}} 
\label{sec:prelim}

We review three existing approaches to \gancond{}: (1) discarding UGAN and training a \cg{} from scratch, (2) fine-tuning~\citep{wangl2018transfer,wang2020mineGAN}, and (3) input reprogramming~\citep{engel2018latent,lee2020}.
In particular, 
we analyze their advantages, drawbacks, and requirements for \gancond{}.  
Sec.~\ref{sec:prelim_comparison} summarizes our \td{high-level} comparison of their performances in various training settings of labeled data and computing resources.

\subsection{Learning \cg{} without using UGAN}
One can discard the given UGAN and apply the existing CGAN training algorithms on the labeled dataset to learn a CGAN.
This approach is the most straightforward approach for \gancond{}. 

We can group \cg{} training algorithms by strategies of incorporating label information into the training procedure.
The first strategy is to concatenate or embed labels to inputs~\citep{mirza2014conditional,denton2015deep} or to middle-layer features~\citep{reed2016generative,dumoulin2016adversarially}. 
The second strategy, which generally achieves better performance, is to design the objective function to incorporate conditional information.
For instance, \ag{}~\citep{odena2017conditional} adds a classification loss term to the original discriminator objective via an auxiliary classifier.
Recent algorithms adopt label projection in the discriminator~\citep{miyato2018cgans} by linearly projecting the embedding of the label vector into the feature vector.
Based on the projection-based strategy, \ctg{}~\citep{kang2020contragan} further utilizes the data-data relation between samples to achieve better quality and diversity in data generation.

Directly applying such CGAN algorithms for \gancond{} may achieve state-of-the-art conditioning performances if sufficient labeled data and computing resources are available.
However, this is not always the case in practice, and the approach also has some other drawbacks.
The scarcity of labeled data in practice can severely degrade the performance of CGAN algorithms~\citep{shahbazi2022collapse}.
Though multiple techniques~\citep{lucic2019high,karras2020training} were proposed to overcome this scarcity issue, mainly by using data augmentation to increase labeled data, these techniques are orthogonal to our work, and we consider only the standard algorithms. 
Furthermore, training \cg{} from scratch is notoriously challenging and expensive.
Researchers observed various factors that make CGAN training difficult, such as instability~\citep{kodali2017convergence,odena2019open}, mode collapse~\citep{salimans2016improved,odena2019open}, or mode inventing~\citep{salimans2016improved}.
Also, the training usually entails enormous resources in terms of time, computation, and memory, which are not available for some settings, such as mobile or edge computing.
For instance, training BigGAN~\citep{brock2018large} takes approximately 15 days on a node with $8$x NVIDIA Tesla V100/32GB GPUs~\citep{biggan_training}.

Additionally, some CGAN training algorithms have inductive biases, which may cause failures in learning the true distributions.
The following lemmas investigate several failure scenarios of the two most popular conditioning strategies -- \ag{}~\citep{odena2017conditional} and \pg{}~\citep{miyato2018cgans}, and \td{we defer the formal statements and proofs to Appendix~\ref{app:proof_prelim}}.

\paragraph{Failures of auxiliary classifier conditioning strategy}
The discriminator and generator of \ag{} learn to maximize $\lambda L_C + L_S$ and $\lambda L_C - L_S$, respectively.
\td{Here, $L_S$ models the log-likelihood of samples belonging to the real data, $L_C$ models the log-likelihood of samples belonging to the correct classes, and $\lambda$ is a hyperparameter balancing the two terms.}
In this game, $G$ might be able to maximize $\lambda L_C - L_S$ by simply learning \textit{a biased distribution} (increased $L_C$) at the cost of compromised generation quality (increased $L_S$).
Our following lemmas show that \ag{} indeed suffers from this phenomenon, for both non-separable datasets (Lemma~\ref{ex:non_separable}) and separable datasets (Lemma~\ref{ex:separable}).
We also note that the failure in non-separable datasets has been previously studied in~\citep{shu2017ac}, while the one with separable datasets has not been shown before in the literature.
\begin{lemma} [\ag{} provably fails on a non-separable dataset]
\label{ex:non_separable}
Suppose {that} the data follows a Gaussian mixture distribution with a known location but unknown variance, and the generator is a Gaussian mixture model parameterized by its variance.
\td{Assume the perfect discriminator. 
For some values of $\lambda$, the generator's loss function has strictly suboptimal local minima, thus gradient descent-based training algorithms fail to find the global optimum.}
\end{lemma}

\begin{lemma} [{\ag{} provably fails} on a separable dataset]
 \label{ex:separable}
Suppose that the data are vertically uniform in $2$D space: conditioned on $y \in \{\pm 1\}$, $\rvx=(d \cdot \ry, \ru)$ with some $d>0$ and $\ru$ is uniformly distributed in $[-1,1]$. 
An auxiliary classifier is a linear classifier, passing the origin. 
\td{Assume the perfect discriminator.
For some single-parameter generators, gradient descent-based training algorithms converge to strictly suboptimal local minima.}
\end{lemma}

\paragraph{Failures of projection-based conditioning strategy}
The objective function of a projection-based discriminator~\citep{miyato2018cgans} measures the orthogonality between the data feature vector and its class-embedding vector in the form of an inner product.
Thus, even when the generator learns the correct conditional distributions, it may continue evolving to further orthogonalize the inner product term if the class embedding matrix is not well-chosen. 
This behavior of the generator can result in learning \td{an} inexact conditional distribution, illustrated in the following lemma.
\begin{lemma} [ProjGAN provably fails on \td{a} two-class dataset]
Consider a simple projection-based CGAN with two equiprobable classes.
\td{With some particular parameterizations of the discriminator}, there exist bad class-embedding vectors that encourage the generator to deviate from the exact conditional distributions.
\end{lemma}

\subsection{Fine-tuning UGAN into CGAN}
The fine-tuning approach aims to adjust the provided \ugen{} into a \cgen{} using the labeled data.
Previous works~\citep{wangl2018transfer,wang2020mineGAN} studied fine-tuning GANs for knowledge transfer in GANs.
We can further adapt these approaches for \gancond{}. 
Transfer\gan{}~\citep{wangl2018transfer} introduces a new framework for transferring GANs between different datasets.
They propose to fine-tune both the pretrained generator and discriminator networks.
Mine\gan{}~\citep{wang2020mineGAN} later suggests fixing the generator and training an extra network (called the miner) to optimize the latent noises for the target dataset before fine-tuning all networks.
We note that in the \gancond{} setting, the pretrained discriminator is \textit{not} available. 
Furthermore, to be applicable for \gancond{}, Transfer\gan{} and Mine\gan{} frameworks may require further modifications of unconditional generators' architecture to incorporate the conditional information.
Unlike fine-tuning approaches, our method freezes the $G$ and does not modify its architecture or use the pretrained unconditional discriminator.

Compared to the full CGAN training, fine-tuning approaches usually require much less training time and amount of labeled data while still achieving competitive performance. %
However, it is not clear how one will modify the architecture of the \ugen{} for a conditional one.
Also, fine-tuning techniques are known to suffer from catastrophic forgetting~\citep{li2017learning}, which may severely affect the quality of generated samples from well-trained unconditional generators.

\subsection{Input reprogramming}
The overarching idea of \irep{} is to keep the well-trained UGAN intact and add a controller module to the input end for controlling the generation.
In particular,
this approach aims to repurpose the \ugen{} into a \cgen{} by only learning the class-conditional latent vectors.

Input reprogramming can be considered as neural reprogramming~\citep{elsayed2019} applied to generative models.
Neural reprogramming (or adversarial reprogramming)~\citep{elsayed2019} has been proposed to repurpose pretrained classifiers for alternative classification tasks by just preprocessing their inputs and outputs.
Recent works extend neural reprogramming to different settings and applications, such as the discrete input space for text classification tasks~\citep{neekhara2018adversarial} and the black-box setting for transfer learning~\citep{tsai2020transfer}.
More interestingly, it has been shown that 
large-scale pretrained transformer models
can also be efficiently repurposed for various downstream tasks, with most of the pretrained parameters being frozen~\citep{lu2021pretrained}.
On the theory side, the recent work~\citep{yang2021voice2series} shows that the risk of a reprogrammed classifier can be upper bounded by the sum of the source model's population risk and the alignment loss between the source and the target tasks.

Input reprogramming has been previously studied under various generative models and learning settings.
Early works attempt to reprogram the latent space of \gan{} with pretrained classifiers~\citep{nguyen2017plug} or the latent space of variational autoencoder (VAE)~\citep{engel2018latent} via \cg{} training.
Flow-based models can also be repurposed for the target attributes via posterior matching~\citep{gambardella2019transflow}.
The recent work~\citep{nie2021controllable} proposes to formulate the conditional distribution as an energy-based model and train a classifier on the latent space to control conditional samples of unconditional generators.
However, their training requires manually labeling latent samples and a new sampling method based on ordinary differential equation solvers.
We also notice that our work has a different assumption on the input as the only source of supervision in \gancond{} comes from the labeled data. 
Recently, \grep{}~\citep{lee2020} via \cg{} training obtains high-quality conditional samples with significant reductions of computing and label complexity.
\grep{} is our closest \irep{} algorithm for \gancond{}.

For the advantages, \irep{} methods obtain the competitive performances in the conditional generation on various image datasets, even with limited labeled data.
\td{Input reprogramming} significantly saves computing and memory resources as the approach requires only an extra lightweight network for each inference.
However, the existing algorithm of \irep{} for \gancond{}, \grep{}~\citep{lee2020}, still underperforms the latest \cg{} algorithms on complex datasets. 
In addition, it focuses solely on one condition per time, leading to the huge memory given a large number of classes, which will be detailed in the next section.

Our proposed method (\rg{}) makes \td{significant} improvements on \grep{}.
We improve the conditioning performance with invertible networks and a new loss based on the Positive Unlabeled learning approach.
We reduce the memory footprints by sharing weights between class-conditional networks.
Our study will exhibit that input reprogramming can generate sharper distribution even with small amounts of supervision, leading to consistent performance improvements.
We also explore the robustness of input reprogramming in the setting of imbalanced supervision and noisy supervision, which have not been discussed in the literature yet.

\subsection{Summary of comparison}
\label{sec:prelim_comparison}
\begin{table*}
    \vspace{-1em}
    \caption{\textbf{\td{A high-level} comparison of \gancond{} approaches {under various} settings (\checkmark=~poor, \checkmark\checkmark=~okay, \checkmark\checkmark\checkmark=~good)}.
    When labeled data and computing resources are both \td{sufficient}, we can discard the \ugen{} and directly apply \cg{} algorithms on the labeled data to achieve the best performances.
    When the amount of labeled data is large, but resources are restricted, \td{the fine-tuning approach} becomes a good candidate for \gancond{}.
    However, when both labeled data and resources are restricted, \irep{} becomes the best solution for \gancond{}.
    Our method \rg{} improves both the conditioning performance and the memory scalability of the existing \irep{} algorithm (\grep{}).
    }

    \centering
    \label{tab:gancond}
    \begin{tabular}{ccccc}
        \hline
        \toprule[1pt]
        \multirow{2}{*}{\td{Setting}}  & \multirow{2}{*}{Learning CGAN w/o UGAN} & \multirow{2}{*}{Fine-tuning} & \multicolumn{2}{c}{Input reprogramming}  \\
        & & & \textit{\grep{}}& \textit{\rg{}} (ours)\\
        \hline 
        \thead{Large labeled data} & \checkmark\checkmark\checkmark & \checkmark\checkmark & \checkmark & \checkmark\textbf{\checkmark} \\
        \thead{Limited labeled data} & \checkmark & \checkmark\checkmark & \checkmark\checkmark & \checkmark\checkmark\textbf{\checkmark} \\
        \thead{Limited computation} & \checkmark &\checkmark\checkmark & \checkmark\checkmark\checkmark & \checkmark\checkmark\checkmark \\
        \hline
        \toprule[1pt]
    \end{tabular}
\end{table*}
    
\td{
Table~\ref{tab:gancond} summarizes our high-level comparison of the analyzed approaches.
For \irep{}, we distinguish between the existing approach \grep{} and our proposed \rg{}.
First, when both the labeled data and computation resources are sufficient, using \cg{} algorithms without UGAN probably \td{achieves} the best conditioning performance, followed by the fine-tuning and \rg{} methods.
However, when the labeled data is scarce (e.g., less than $10$\% of the unlabeled data), directly training \cg{} from scratch may suffer from the degraded performance.
Methods that reuse UGAN (fine-tuning, \grep{}, \rg{}) are \td{better alternatives} in this scenario, and our experimental results show that \rg{} achieves the best performance among these methods.
Furthermore, if the computation resources are more restricted, \irep{} approaches are the best candidates as they gain significant computing savings while achieving good performances.
}
\section{Improved Input Reprogramming for \gancond{}}
\label{sec:inrep}

We first justify the use of the general \irep{} method for \gancond{} and analyze issues in the design of the current \irep{} framework (Sec.~\ref{sec:inrep_revisit}).
To address these issues, we propose a novel \rg{} framework in Sec.~\ref{sec:inrep_model}.
We theoretically show the optimality of \rg{} in learning conditional distributions in Sec.~\ref{sec:inrep_analysis}.
Algorithm~\ref{alg:repgan} presents the full \rg{} training algorithm.

\subsection{Revisiting \irep{} for \gancond{}}
\label{sec:inrep_revisit}
The idea of \irep{} is to repurpose the pretrained \ugen{} $G$ into a \cgen{} simply by preprocessing its input without making \emph{any} change to $G$.
\td{Intuitively, freezing the \ugen{} does not necessarily limit the \irep{}'s capacity of conditioning GANs.
Consider a simple setting with the perfect generator $G: G(\rvz) \eqd \rvx$, where $\rvz$ and $\rvx$ are discrete random variables. 
For any discrete random variable $\ry$, possibly dependent on $\rvx$, we can construct a random variable $\rvz_{y}$ such that $G(\rvz_y) \eqd \rvx|\ry=y$ by redistributing the probability mass of $\rvz$ into the values corresponding to the target $\rvx|\rvy=y$.
We provide Proposition~\ref{prop:1} in \td{Appendix~\ref{app:proof_31}} as a formal statement to illustrate our intuition in the general setting with continuous input spaces.
}

The \irep{} approach aims at learning conditional noise vectors $\{\rvz_y\}$ such that $G(\rvz_y) \eqd \rvx|y$ for all $y$. 
We can view this problem as an implicit generative modeling problem. 
For each $y$, we learn a function $M_y(\cdot) = M(\cdot, \rvy)$ such that $\rvz_y \eqd M_y(\rvu)$ and $G(\rvz_y) \eqd \rvx|y$ for standard random noise $\rvu$ via \gan{} training with a discriminator $D_y$.

The existing algorithm for \irep{}, \grep{}~\citep{lee2020}, proposes to separately learn $M_y, D_y$ for each value of $y$, where $M_y$ is a neural network, called modifier network.
Fig.~\ref{fig:differences-comparison} (left) visualizes the design of the \grep{} framework for \td{an $n$}-class condition.
\grep{} requires \td{$n$} pairs of modifier and discriminator networks, each per class condition. 
This design has several issues in learning to condition GANs.

\paragraph{Issues of \grep{} in \gancond{}}
\textit{First,} there might be overlaps between the latent regions learned by modifier $M_y$ and $M_{y'}$ ($y'\neq y$) due to the imperfection in learning. 
These overlaps might result in incorrect class-conditional sampling.
\textit{The second issue} is the training dynamic of conditional discriminators. 
Assume the perfect pretrained \ugen{}, for instance, $G(\rvz) \eqd \rvx$ with random Gaussian $\rvz$ and two equiprobable classes for simplicity.
Then, we can view $G(\rvz)$ as a mixture of $\rvx|\ry=0$ and $\rvx|\ry=1$.
At the beginning of $(M_y, D_y)$ training, a fraction of the generated data is distributed as $\rvx|\ry=y$, which we desired, but labeled as fake and rejected by $D_y$.
As $M_y(\rvu)$ approaches $\rvz_y$, a larger fraction of desired samples will be wrongly labeled as fake. 
This phenomenon can cause difficulty in reprogramming, especially under the regime of low supervision.
\textit{The last issue} is the vast memory when the number of classes is large because \grep{} trains separately each pair of $(M_y, D_y)$ per condition. %

\subsection{\rg{}: An improved input reprogramming algorithm}
\label{sec:inrep_model}
We propose \rg{} to address identified issues of \grep{}.
Shown in Fig.~\ref{fig:model} is the design of our \rg{} framework.
We adopt the weight-sharing design for modifier and discriminator networks to improve memory scalability.
Also, we design our modifier network to be invertible to prevent the overlapping issue in the latent space.
Lastly, we derive a new loss for more stable training based on the recent Positive-Unlabeled (PU) learning framework~\citep{guo2020positive}.
We highlight differences of \rg{} over previous \grep{} in Fig.~\ref{fig:differences-comparison}.
\begin{figure*}
    \vspace{-2mm}
    \centering
 	\includegraphics[width=0.8\textwidth]{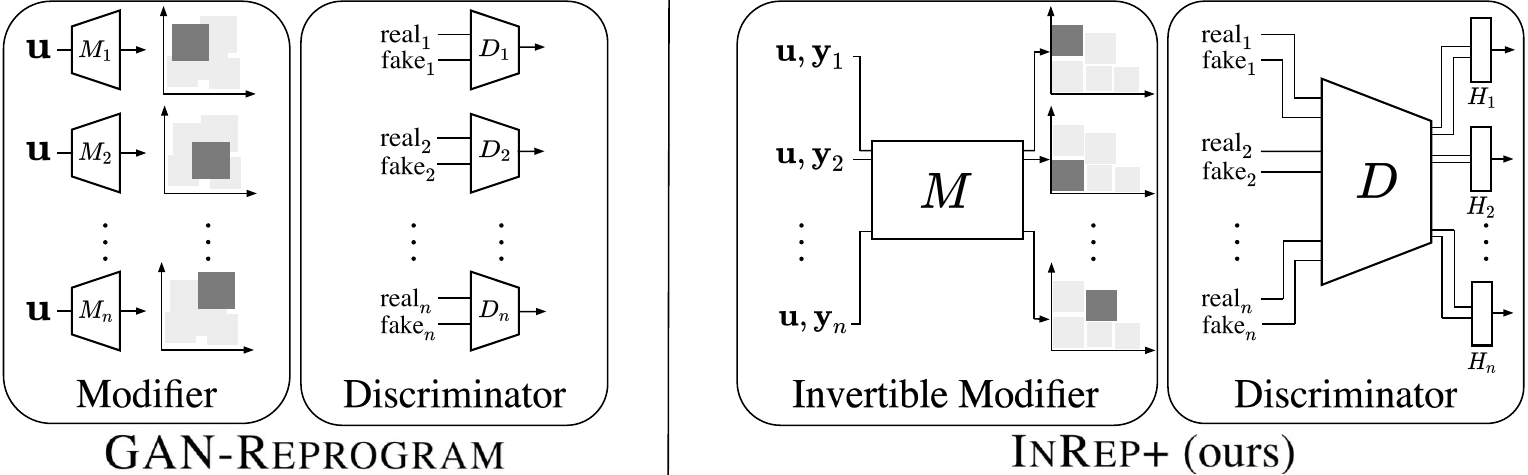}
    \caption{
    \textbf{\td{Highlights of design differences} between \grep{}~\citep{lee2020} and \rg{} (ours).}
    \grep{} (left) learns a pair of modifier $M_i$ and discriminator $D_i$ per each of \td{$n$} conditions.
    \rg{} (right) makes three improvements over \grep{}: i) using a single conditional modifier $M$ and sharing weights between $\{D_i\}_{1}^n$ to improve memory scalability when $n$ is large, ii) adopting the invertible architecture for modifier network $M$ to prevent  class-conditional latent regions from overlapping, iii) replacing the standard GAN loss with a new PU-based loss to overcome the false rejection issue.}
    \label{fig:differences-comparison}
    \vspace{-1mm}
\end{figure*}

\paragraph{Weight-sharing architectures}
For the modifier, we use only a single conditional modifier network $M$ for all classes. 
For the discriminator, we share a base network $D$ for all classes and design  multiple linear heads $H_y$ on top of $D$, each head for a conditional class, \td{that is}, $D_y = H_y \circ D$.
This design helps maintain the separation principle of \irep{} while significantly reducing the number of trainable parameters and weights needed to store.

\paragraph{Invertible modifier network} 
The modifier network $M$ learns a mapping $\mathbb{R}^{d_\rvz} \to \mathbb{R}^{d_{\rvz_y}}$, where $\rvz$ is the input noise and $\rvz_{y}$ is the $y$-conditional noise.
The input noise $\rvz$ is an aggregation of a random noise vector $\rvu$ and the label $y$ in the form of label embedding~\citep{miyato2018cgans}.
We use a learnable embedding module $E$ to convert each discrete label into a continuous vector with dimension $d_\rvy$.
The use of label embedding instead of one-hot encoding has been shown to help GANs learn better~\citep{miyato2018cgans}.
Here, $d_{\rvz} = d_\rvu + d_\rvy$ with $d_{\rvz}, d_\rvu$, and $d_\rvy$ being dimensions of $\rvz$, $\rvu$, and $\rvy$ respectively.
In this work, we use the concatenation function as the aggregation function.
To prevent outputs of modifiers on different classes from overlapping, 
we design the modifier to be invertible using the architecture of invertible neural networks~\citep{jacobsen2018irevnet,behrmann2019invertible,song2019mintnet}.
Intuitively, the invertible modifier maps different random noises to different latent vectors, guaranteeing uniqueness.
Notably, the use of invertibility leads to a dimension gap between the random noise $\rvu$ and the conditional noise $\rvz_y$ as $d_\rvu < d_{\rvz} = d_{\rvz_y}$. 
However, the small dimension gap does not significantly affect the performance of the modifier network on capturing the true distribution of the conditional latent because the underlying manifold of complex data usually has a much lower dimension in most practical cases~\citep{fefferman2016testing}.
Modifier $M$ needs to be sufficiently expressive to learn the target noise partitions while maintaining the computation efficiency.

\paragraph{Training conditional discriminator with Positive-Unlabeled learning} 
\td{Assume that $G$} generates $y$-class data with probability $p_{\text{class}}(y)$.
At the beginning of discriminator training, $p_{\text{class}}(y)$ fraction of generated samples are high-quality and in-class but labeled as fake by the discriminator.
This results in the training instability of the existing \gan{} reprogramming approach~\citep{lee2020}.
To deal with this issue, we view the discriminator training through Positive-Unlabeled (PU) learning lens.
PU learning has been studied in binary classification, where only a subset of the dataset is labeled with one particular class, and the rest is unlabeled \citep{de1999positive,scott2009novelty,du2014analysis}.
Considering the generated data as unlabeled, we can cast the conditional discriminator learning as a PU learning problem: classifying positive (in-class and high-quality) data from unlabeled (generated) data. 

Now, we consider training $(M_y,D_y)$ for a fixed $y$.
As discussed, the distribution of generated data $G(M_y(\rvu))$ can be viewed as a mixture of $p_{\text{data}}(\vx|y)$ and the residual distribution, say $p_\text{gf}(\vx)$, where `g' stands for generated and `f' means that they should be labeled as fake because they are out-class.
That is, letting $\pi_y$ be the fraction of $y$-class data among generated data, we have $p_{G(M_y(\rvu))}(\vx) = \pi_y \cdot p_{\text{data}}(\vx|y) + (1-\pi_y) \cdot p_\text{gf}(\vx)$.
Given this decomposition, the ideal discriminator objective function is:
\begin{small}
\begin{align}
V_y^{PU} &= \mathbb{E}_{\mathbf{x} \sim p_{\text{data}}(\vx|y)}\left[\log D_y(\rvx)\right] + \pi_y \mathbb{E}_{\mathbf{x} \sim p_{\text{data}}(\vx|y)}\left[\log(D_y(\rvx))\right] + (1-\pi_y)\mathbb{E}_{\rvx \sim p_\text{gf}(\vx)}\left[\log(1-D_y(\rvx))\right] \label{eq:loss_PU_1} \\
&= (1+\pi_y) \mathbb{E}_{\mathbf{x} \sim p_{\text{data}}(\vx|y)}\left[\log D_y(\rvx) \right] + \mathbb{E}_{\rvu \sim p_{\rvu}(\vu)}\left[\log(1-D_y(G(M_y(\rvu))))\right] %
- \pi_y \mathbb{E}_{\mathbf{x} \sim p_{\text{data}}(\vx|y)}\left[ \log(1-D_y(\rvx))\right] \label{eq:PU2}
\end{align}
\end{small}
where the last equation follows from $(1-\pi_y) p_\text{gf}(\vx) = p_{G(M_y(\rvu))}(\vx) - \pi_y p_{\text{data}}(\vx|y)$.

\begin{wrapfigure}{r}{0.4\textwidth}
    \vspace{-4mm}
	\centering
	\includegraphics[width=0.85\linewidth]{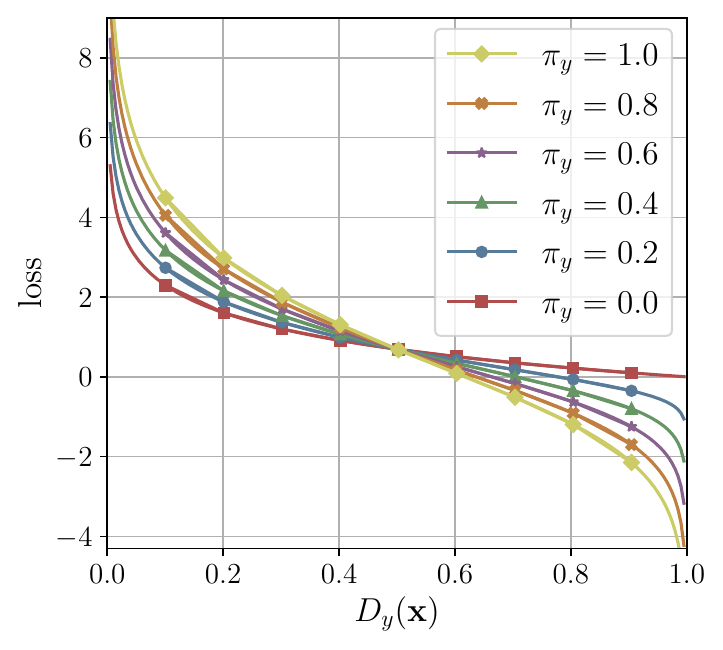}
	\vspace{-3mm}
	\caption{\textbf{Visualization of \rg{}'s  discriminator loss on real data.}
	We visualize the loss %
	with different values of $\pi_y$ (in $[0, 1]$) and $D_y(x)$ (in $[0, 1]$).
	Compared to the standard negative log-loss ($\pi_{y}=0$), the new loss ($\pi_{y}>0$)  more strongly encourages $D_{y}$ to predict values near $1$ for positive samples. }
	\label{fig:newloss}
\end{wrapfigure}

The discriminator loss is $-V_y^{PU}$.
For real data $\mathbf{x} \sim p_{\text{data}}(\vx|y)$, 
discriminator minimizes a new loss $-(1 + \pi_y) \log(D_y(\rvx)) + \pi_y \log(1 - D_y(\rvx))$ instead of the standard negative log-loss $-\log(D_y(\rvx))$.
Fig.~\ref{fig:newloss} visualizes this new loss with various values of $\pi_y$.
Compared to the standard loss ($\pi_y = 0$), our loss function strongly encourages the discriminator to assign positively labeled data higher scores that are very close to $1$. 
Also, when the prediction is close to $1$, the loss takes an unboundedly large negative value, indirectly assigning more weights to the first term in (\ref{eq:PU2}). 
If $\pi_y$ is close to $1$, the relative weights given to the real samples become even higher. 
An intuitive justification is that the loss computed on the generated data becomes less reliable when $\pi_y$ gets larger, so one should focus more on the loss computed on the real data.

Note that the last term in (\ref{eq:loss_PU_1}) is always negative as $D_y$ outputs a probability value, but its empirical estimate, shown in the sum of the last two terms in (\ref{eq:PU2}), may turn out to be positive due to finite samples.
We correct this by clipping the estimated value to be negative as shown below:
\begin{equation*}
  \resizebox{\hsize}{!}{
 $V_y^{PU} = (1+\pi_y) \widehat{\mathbb{E}}_{\mathbf{x} \sim p_{ \text{data}}(\vx|y)}\left[\log D_y(\rvx) \right] %
+ \min \left\{ 0, \widehat{\mathbb{E}}_{\rvu \sim p_{\rvu}(\vu)}\left[\log(1-D_y(G(M_y(\rvu))))\right]  - \pi_y \widehat{\mathbb{E}}_{\mathbf{x} \sim p_{\text{data}}(\vx|y)}\left[\log(1-D_y(\rvx))\right] \right\}  
$}  
\end{equation*}
We set $\pi_y=\tfrac{1}{|\sY|}$ initially when the generated data is well-balanced between classes, and gradually increase $\pi$ to $1$ by temperature scaling to capture the increasing fraction of $y$-class data.

\begin{remark}
\td{
In the context of the UGAN training, recent work~\citep{guo2020positive} studies a similar PU-based approach to address the training instability caused by the false rejection issue discussed above.
However, unlike our approach, their approach does not consider the term $\mathbb{E}_{\rvx \sim p_{\text{data}}(\vx|y)}\left[\log D_y(\rvx)\right]$ in their loss function design.
This design makes their loss function invalid when $\pi_y$ is close to $0$, \td{that is}, at the beginning of GAN training.
On the other hand, our loss function is valid for all $\pi_y \in [0, 1]$.
We also note that the false rejection issue is more relevant near the end of the standard GAN training when the generated samples are more similar to the real ones, while it occurs right at the beginning of \gancond{} training.
This difference leads to different uses of the PU-based loss for the standard GAN and \gancond{}.
}
\end{remark}

\setlength{\floatsep}{0.1cm}
  \begin{algorithm}[ht]
	\caption{$\texttt{\rg{}}$}
	\label{alg:repgan}
	\SetKwInput{Input}{Input}
	\SetKwInput{Result}{Result}
	\Input{Pretrained \ugen{} $G$, class-conditional dataset $\sD$ with label set $\mathbb{Y}$, noise dimension $d$,  batch size $m$, \td{learning rates $\alpha, \beta$}, number of discriminator steps $k$, number of training iterations $t$}
	\Result{Class-conditional generator $G'$} 
	Initialize parameters: $\vtheta_M$ for modifier network $M$, $\vtheta_E$ for embedding module $E$, $\vtheta_D$ for discriminator network $D$, and $\{\vtheta_y: y \in \sY\}$ for linear heads $\{H_y: y \in \sY \}$.
	
	Let $D_y=H_y(D(\cdot; \vtheta_D);\vtheta_y)$,  $\td{\vw_{D_y}} = [\vtheta_D, \vtheta_y]$ for each $y\in\sY$, $\vw_{EM} = [\vtheta_M, \vtheta_E]$\\
	\For{$t$ iterations}{
	    \For{$k$ steps}{
		    sample $\{\rvx_i, \ry_i\}_{i=1}^m \sim \sD$, $\{\rvu_i\}_{i=1}^m \sim \mathcal{N}(\mathbf{0}, \mI_d)$ \\
		    $\rvy_i = E(\ry_i; \vtheta_E)$ \td{for all $i$}; 
	    	~~$\tilde{\rvx}_i \leftarrow G(M(\rvu_i, \rvy_i;\vtheta_M))$ \td{for all $i$}\\
	    	\For{each $y \in \sY$}{
		        $V_y^{(i)} \leftarrow V_y^{PU}(\tilde{\rvx}_i, \rvx_i )$ \td{for all $i$}\\ 
		        $\nabla_{\vw_{D_y}} = \text{Adam}(\nabla_{\td{\vw_{D_y}}} \frac{1}{m}\sum_{i=1}^{m}{V_y^{(i)}})$; 
		        ~~$\td{\vw_{D_y}} \leftarrow \td{\vw_{D_y}} + \alpha\nabla_{\vw_{D_y}}$
		    }
	    }
		sample $\{\hat{\rvu}_i\}_{i=1}^m \sim \mathcal{N}(\mathbf{0}, \mI_d)$ \\
		\For{each $y \in \sY$}{
		    $\rvy = E(y; \vtheta_E)$;
    		$\hat{\rvx}_i \leftarrow G(M(\hat{\rvu}_i, \rvy;\vtheta_M))$ \td{for all $i$}; 
    		~~$V_y^{(i)} \leftarrow V_y^{PU}(\tilde{\rvx}_i, \rvx_i )$ \td{for all $i$} \\ 
    		$\nabla_{\vw_{EM}} = \text{Adam}(-\nabla_{\vw_{EM}} \frac{1}{m} 
    		\sum_{i=1}^{m} \td{V_y^{(i)}})$; 
    		~~$\vw_{EM} \leftarrow \vw_{EM} + \td{\beta} \nabla_{\vw_{EM}}$
		}
	}
	$G'(\cdot, \cdot) \leftarrow G \circ M(\cdot, E(\cdot; \vtheta_E);\vtheta_M)$
\end{algorithm}

\subsection{\rg{}’s optimality under \td{ideal setting}}
\label{sec:inrep_analysis}
In this section, we highlight that the GAN's global equilibrium theorem~\citep{goodfellow2014generative} still holds under the PU-learning principle incorporated in the discriminator training.
Specifically, \rg{} attains the global equilibrium if and only if generated samples follow the true conditional distribution. 
This theorem guarantees that \rg{} learns the exact conditional distribution under the ideal training. We provide the proofs for the proposition and the theorem in \td{Appendix~\ref{app:proof_33}}.

\begin{proposition}
Fix the ideal unconditional generator $G$ and arbitrary modifier $M_y$, the optimal discriminator for $y$ is 
\begin{align*}
    D_y^*(M_y(\rvu)) = \frac{(1+\pi_y)p_{\text{data}}(\vx|y)}{(1+\pi_y) p_{\text{data}}(\vx|y) + (1-\pi_y) p_\text{gf}(\vx)}.
\end{align*}
\end{proposition}
The proposition shows that the optimal discriminator $D^*$ learns a certain balance between $p_{\text{data}}(\vx|y)$ and $p_{\text{gf}}(\vx)$.
Fixing such an optimal discriminator,
we can prove the equilibrium theorem. 
\begin{theorem}[Adapted from \citep{goodfellow2014generative}] \label{thm:exact_recovery}
When the ideal unconditional generator $G^*$ and discriminator $D^*$ are fixed, the globally optimal modifier is attained if and only if $p_\text{gf}(\vx) = p_{G(M_y(\rvu))}(\vx) = p_{\text{data}}(\vx|y)$. 
\end{theorem}

\section{Experiments}

\label{sec:experiment}

In this section, we first describe our experiment settings, including datasets, baselines, evaluation metrics, network architectures, and the training details.
In Sec.~\ref{sec:exp_visual}, we present conditioned samples from \rg{} on various data.
In Sec.~\ref{sec:exp_supervision}, we quantify the learning ability of \rg{} varying the amount of labeled data on different datasets.
In Sec.~\ref{sec:exp_robustness}, we evaluate the robustness of \rg{} in settings of class-imbalanced and label-noisy labeled data.
We also conduct an ablation study on the role of \rg{}'s components (Sec.~\ref{sec:exp_ablation}).
\td{Our code and pretrained models are available at \url{https://github.com/UW-Madison-Lee-Lab/InRep}.}

\paragraph{Datasets and baselines} 
We use a synthetic \dg{} dataset (detailed in Sec.~\ref{sec:exp_visual}) and various real datasets: \dm{}~\citep{deng2012mnist}, \dc{}~\citep{Krizhevsky09learningmultiple}, \dhc{}~\citep{Krizhevsky09learningmultiple}, Flickr-Faces-HQ (FFHQ)~\citep{karras2019style}, CelebA~\citep{liu2015faceattributes}.
Our baselines are \grep{}~\citep{lee2020}, fine-tuning~\citep{wangl2018transfer,wang2020mineGAN}, and three popular standard \cg{}s (\ag{}, \pg{}, and \ctg{}).
\ftg{} is a combined approach of TransferGAN~\citep{wangl2018transfer} and  MineGAN~\citep{wang2020mineGAN}, which first learns a discriminator and a miner network given the condition set~\citep{wang2020mineGAN}, then fine-tunes all networks using \ag{} loss.

\paragraph{Evaluation metrics} 
We use the popular Fr\'{e}chet Inception Distance (\fid{})~\citep{heusel2018gans} and recall~\citep{kynkaanniemi2019improved} to measure the quality and diversity of learned distributions.
\fid{} measures the Wasserstein-2 (Fr\'{e}chet) distance between the learned and true distributions in the feature space of the pretrained Inception-v3 model.
Lower \fid{} indicates better performance.
To measure the conditioning performance, we use \cfid{}~\citep{miyato2018cgans} and Classification Accuracy Score (CAS)~\citep{ravuri2019classification,Lesort_2019,shmelkov2018good}.
Specifically, \cfid{} is the average of \fid{} scores measured separately on each class, and \cas{} is the testing accuracy on the real data of the classifier trained on the generated data.

\paragraph{Network architectures and training}
\td{Our architectures and configurations of networks are mainly based on} GANs' best practices~\citep{lucic2019high,kang2020contragan,kurach2019large,asveegan17}.
We adopt the same network architectures for all models.
For \rg{}, 
we set the dimension of label embedding vectors to $10$ for all experiments.
Our modifier network uses the i-ResNet architecture~\citep{behrmann2019invertible}, \td{with three layers for simple datasets (\dg{}, \dm{}, \dc{}) and five layers for more complex datasets (\dhc{}, FFHQ)}.
We design the modifier networks to be more lightweight than generators and discriminators.
For instance, our modifier for \dc{} has only $0.1$M parameters while the generator and the discriminator have $4.3$M and $1$M parameters, respectively.
The latent distribution is the standard normal distribution with $128$ dimensions.
\textit{For training}, 
we employ Adam optimizer with the learning rates of $2\cdot10^{-4}$ and $2\cdot10^{-5}$ for modifier and discriminator networks. 
$\beta_1, \beta_2$ are $0.5, 0.999$, respectively.
We train \num{100000} steps with five discriminator steps before each generator step.
\rg{} takes much less time compared to other approaches.
For instance, training \ag{} and \ctg{} on \dc{} may take up to several hours, while training \rg{} takes approximately half of an hour to achieve the best performance.
\textit{For implementation}, we adopt the widely used third-party \gan{} library~\citep{kang2020contragan} with PyTorch implementations for reliable and fair assessments.
We run our experiments on two RTX8000 (48 GB memory), two TitanX (11GB), two 2080-Ti (11GB) GPUs.
\textit{For pretrained \ugen{}s}, we mostly pretrain unconditional generators on simple datasets (\dg{}, \dm{}, \dc{}, \dhc{}), and 
make use of  large-scale pretrained models for StyleGAN trained on FFHQ.\footnote{\url{https://github.com/rosinality/stylegan2-pytorch}}
We provide more details in \td{Appendix~\ref{sec:app_experiment}}.

\subsection{Performing \gancond{} with \rg{} on various data}
\label{sec:exp_visual}
\begin{figure}
\vspace{-1em}
    \centering
	\includegraphics[width=\linewidth]{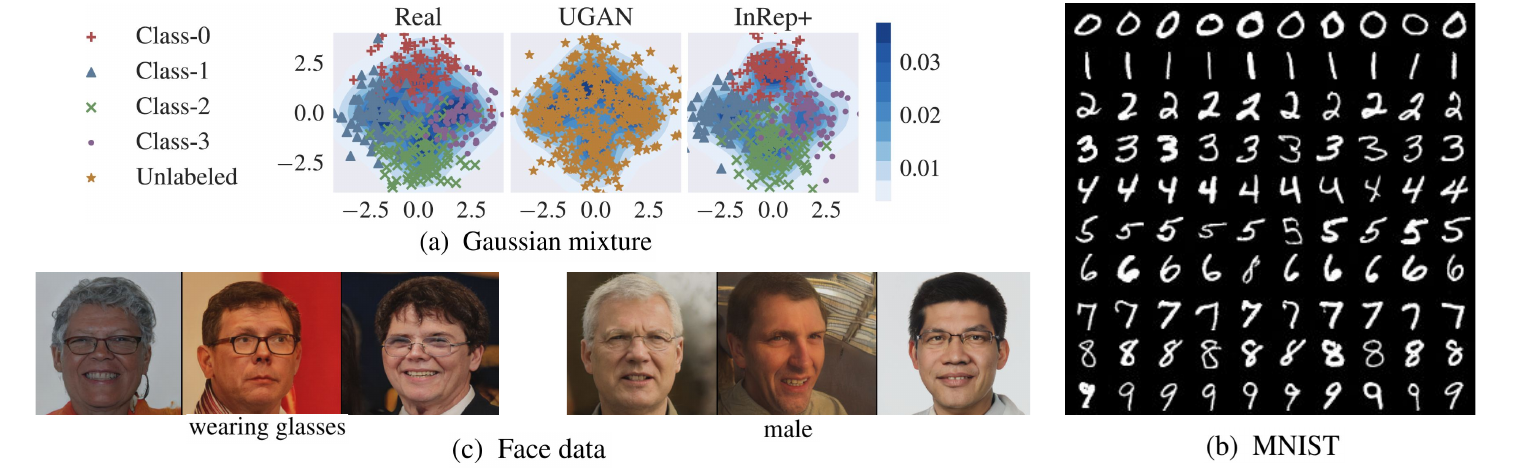}
	\vspace{-1.5em}
	\caption{
	    \textbf{
	    Conditional samples from \rg{} on different datasets.
	    }
	    (a) \textit{\dg{}:} 
	    We synthesize a mixture of four Gaussian distributions on two-dimensional space with \num{10000} samples.
	    Four Gaussian components have unit variance and their means are $(0, 2), (-2, 0), (0, -2), (2, 0)$, respectively. 
	    The synthesized data is visualized in the left figure (Real).
	    The central figure visualizes the distribution learned by unconditional GAN (UGAN). 
	    The right figure (\rg{}) shows the conditional samples from \rg{}.
	    As we can see, the \rg{}'s distribution is highly similar to the real data distribution.
	    Also, it covers all four distribution modes.
	    (b) \textit{\dm{}:}
	    We visualize class-conditional samples of CGANs learned by \rg{}, each row per class.
	    Our samples have correct labels, with high-quality and diverse shapes.
	    (c) \textit{Face data:} 
	    Unconditional GAN model is StyleGAN pretrained on the FFHQ dataset, which contains high-resolution face images.
	    We use CelebA data as the labeled data, with two classes: wearing glasses and male.
	    Most conditional samples are high-quality and with correct labels.
	    We further show conditional samples of \dc{} in Fig.~\ref{fig:cifar10_visual_comparison}.
	}
	\label{fig:visual_repgan}
\end{figure}

Fig.~\ref{fig:visual_repgan} visualizes conditional samples of CGANs learned by \rg{} on different datasets: \dg{}, \dm{}, and the face dataset.
Shown in Fig.~\ref{fig:visual_repgan}a are samples on \dg{} data.
We synthesize a dataset of \num{10000} samples drawn from the \dg{} distribution with four uniformly weighted modes in 2D space, depicted on the left of the figure (Real). 
Four components have unit variance with means $(0, 2), (-2, 0), (0, -2), (2, 0)$, respectively. 
As seen in the figure, generated distribution from \rg{} covers all four modes of data and shares a similar visualization with the real distribution.
For \dm{} (Fig.~\ref{fig:visual_repgan}b), each row represents a class of generated conditional samples.
We see that all ten rows have correct images with high quality and highly diverse shapes.
Similarly, Fig.~\ref{fig:visual_repgan}c shows two classes (\td{wearing glasses and male}) of high-resolution face images, which are generated using \rg{}'s conditional generator.
In this setting, UGAN is the StyleGAN model~\citep{karras2019style} pretrained on the unlabeled FFHQ data, and labeled data is the CelebA data. 
For \dc{}, the last row of Fig.~\ref{fig:cifar10_visual_comparison} shows our conditional samples from \rg{} with different levels of supervision.
These visualizations show  that \rg{} is capable of well-conditioning GANs correctly.

\subsection{\gancond{} with different levels of supervision}
\label{sec:exp_supervision}
\newcommand\visualwidth{1.5in}
\begin{figure}
\centering
\resizebox{\textwidth}{!}{
\begin{tabular}{ccccccc}
\toprule[1pt]
\centering
\td{Method} & Class & 1\% & 10\% & 20\% & 50\% & 100\% \\
\toprule[1pt]
{CGAN (\ag{})} \vspace{1mm} & \makecell{airplane\\ \\ automobile\\  \\bird} &
 {\includegraphics[width = \visualwidth, valign=c]{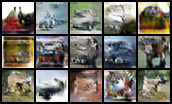}}&
{\includegraphics[width = \visualwidth, valign=c]{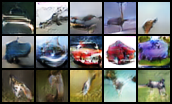}} &
{\includegraphics[width = \visualwidth, valign=c]{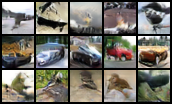}} &
{\includegraphics[width = \visualwidth, valign=c]{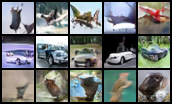}} &
{\includegraphics[width = \visualwidth, valign=c]{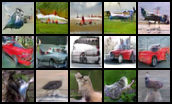}} \\ 
\toprule[0.5pt]

CGAN (\pg{}) \vspace{1mm}& \makecell{airplane\\ \\ automobile\\  \\bird} & \includegraphics[width = \visualwidth, valign=c]{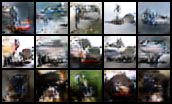} &
\includegraphics[width = \visualwidth, valign=c]{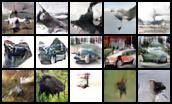} &
\includegraphics[width = \visualwidth, valign=c]{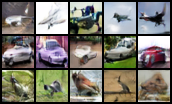} &
\includegraphics[width = \visualwidth, valign=c]{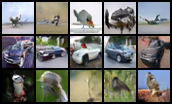} &
\includegraphics[width = \visualwidth, valign=c]{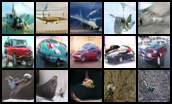} \\
\toprule[0.5pt]

CGAN (\ctg{}) \vspace{1mm} & \makecell{airplane\\ \\ automobile\\  \\bird}
&\includegraphics[width = \visualwidth, valign=c]{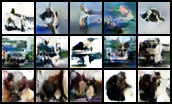} &
\includegraphics[width = \visualwidth, valign=c]{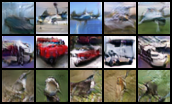} &
\includegraphics[width = \visualwidth, valign=c]{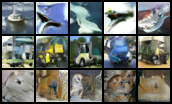} &
\includegraphics[width = \visualwidth, valign=c]{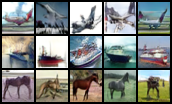} &
\includegraphics[width = \visualwidth, valign=c]{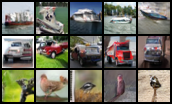} 
\\
\toprule[0.5pt]

\ftg{}\vspace{1mm} & \makecell{airplane\\ \\ automobile\\  \\bird}
&\includegraphics[width = \visualwidth, valign=c]{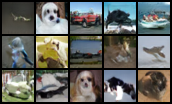} &
\includegraphics[width = \visualwidth, valign=c]{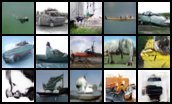} &
\includegraphics[width = \visualwidth, valign=c]{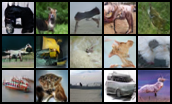} &
\includegraphics[width = \visualwidth, valign=c]{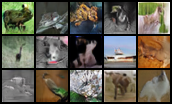} &
\includegraphics[width = \visualwidth, valign=c]{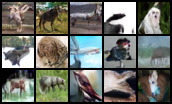} 
\\
\toprule[0.5pt]

\grep{}\vspace{1mm} & \makecell{airplane\\ \\ automobile\\  \\bird}
&\includegraphics[width = \visualwidth, valign=c]{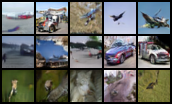} &
\includegraphics[width = \visualwidth, valign=c]{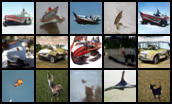} &
\includegraphics[width = \visualwidth, valign=c]{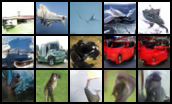} &
\includegraphics[width = \visualwidth, valign=c]{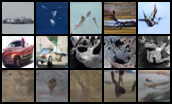} &
\includegraphics[width = \visualwidth, valign=c]{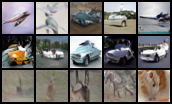} 
\\
\toprule[0.5pt]

\rg{} (ours) \vspace{1mm} & \makecell{airplane\\ \\ automobile\\  \\bird}
&\includegraphics[width = \visualwidth, valign=c]{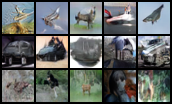} &
\includegraphics[width = \visualwidth, valign=c]{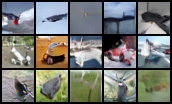} &
\includegraphics[width = \visualwidth, valign=c]{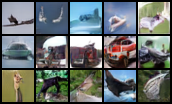} &
\includegraphics[width = \visualwidth, valign=c]{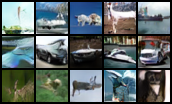} &
\includegraphics[width = \visualwidth, valign=c]{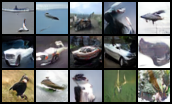} 
\\
\toprule[1pt]
\end{tabular}
}
\caption{
    \textbf{\dc{} conditional images from \gancond{} methods}.
    We compare the samples in terms of quality and class-correctness, varying amounts of labeled data.
    A row of each figure represents images conditioned on one of three classes: airplane, automobile, bird (top to bottom).
    At $1$\% and $10$\%, samples from \ag{}, \pg{}, and \cg{} are blurrier and in lower resolutions than those of methods reusing UGAN (\rg{}, fine-tuning, and \grep{}). 
    Furthermore, \rg{}'s samples have both high quality and correct labels while some samples from \grep{} and fine-tuning methods have incorrect classes.
    When more labeled data is provided (20\%, 50\%, 100\%), \cg{} algorithms (\ag{}, \pg{}, and \ctg{}) gradually synthesize samples with higher quality and correctly conditioned classes.
    Noticeably, while all samples of \ctg{} models have high quality, some of them are mistakenly conditioned.
    This behavior probably causes higher \cfid{} scores in spite of low \fid{} scores.
    Among methods reusing UGANs, most samples from \rg{} have both high quality and correct classes, while \grep{} and the fine-tuning method show more wrong-class conditional samples.
}
\label{fig:cifar10_visual_comparison}
\end{figure}

\begin{table*}[t]
\vspace{-3mm}
\renewrobustcmd{\bfseries}{\fontseries{b}\selectfont}
\sisetup{detect-weight,mode=text,group-minimum-digits =4}
\caption{
\textbf{\fid{} and \cfid{} evaluations.}
We compare \rg{} with baselines on \dc{} and \dhc{} under various amounts of labeled data ($\rx$\% on \dc{} means that $\rx$\% of labeled \dc{} is randomly selected).
When the labeled dataset is small (1\%, 10\%), \rg{} outperforms baselines, shown in the smallest \fid{} and \cfid{} scores over all datasets. 
As more labeled data is given, directly applying \cg{}s results in better performances (lower \fid{} and \cfid{}). 
\rg{} maintains small \fid{} gaps to best models. 
\textit{Values of \grep{} and standard deviations on \dhc{} are not obtained due to the computing limitation.}}
\label{tab:fid_supervision}
    \centering
    \vspace{0.2em}
\resizebox{\textwidth}{!}
{
\begin{tabular}{ccS[table-format=2.2]*{4}{S[table-format=4.2]}cS[table-format=3.2]*{4}{S[table-format=5.2]}c}%
\toprule[1pt]
\centering
\multirow{2}{*}{Dataset}&
\multirow{2}{*}{Method}&
\multicolumn{6}{c}{\fid{}~($\downarrow$)}&
\multicolumn{5}{c}{\cfid{}~($\downarrow$)}\\
\cline{3-7}
\cline{9-14}
&&{1\%} & {\hspace*{2em}10\%} & {\hspace*{2em}20\%} & {\hspace*{2em}50\%} & {\hspace*{2.5em}100\%~~~}&\multicolumn{1}{c}{}& {1\%} & {\hspace*{2em}10\%} & {\hspace*{2em}20\%} & {\hspace*{2em}50\%} & {\hspace*{2em}100\%} &\\
\toprule[1pt]

\multirow{6}{*}{\dc{}}&CGAN (\ag{})& 69.01\std{0.28}&    30.55\std{0.45}&       21.50\std{0.14}&       16.10\std{1.37}&     12.38\std{2.76} &&
141.71\std{0.33}&	88.33\std{0.15}&	73.47\std{0.55}&	68.24\std{2.47}&	54.86\std{8.65}&\\

&CGAN (\pg{})&67.54\std{0.22}&30.35\std{0.88}& 21.21\std{ 0.14}&	14.81\std{ 0.11}&	\bfseries 11.86\std{ 1.86}&&
124.35\std{0.62}&	84.41\std{1.91}&	67.13\std{0.23}&	\bfseries 58.77\std{0.27} &	\bfseries 52.16\std{0.86}&  \\

&CGAN (\ctg{})& 43.11\std{3.28}&	29.44\std{1.73}&	25.53\std{0.49}&	\bfseries 12.44\std{0.16} &	12.19\std{0.14}&&
271.79\std{6.56}&	155.29\std{7.12}&	145.07\std{8.37}&	132.97\std{8.90}&	132.43\std{0.44}&\\

&\ftg{}& 22.01\std{0.19}&	15.97\std{0.70}&	16.13\std{0.17}&	15.40\std{0.58}&	13.98\std{1.12}&&
117.85\std{0.16}&	108.00\std{0.55}&	99.44\std{0.19}&	92.02\std{11.36}&	90.11\std{6.29}&\\

&\grep{}& 23.12\std{3.47}&	22.62\std{0.33}&	19.81\std{0.12}&	16.27\std{0.16}&	15.45\std{0.37}&& 
114.51\std{1.33}&	78.04\std{5.51}&	74.40\std{6.47}&	67.52\std{5.05}&	68.32\std{5.14}&\\
\cline{2-14}

&\rg{} (ours)&\bfseries 16.39\std{1.37} &	\bfseries 14.52\std{0.97} &\bfseries 	14.84\std{1.22} &	13.45\std{3.18}&	12.99\std{1.21}
&& \bfseries 76.24\std{3.35}&	\bfseries 65.45\std{3.25}&	\bfseries 62.19\std{3.22}&	61.81\std{1.34}&	59.62\std{2.21}&\\

\toprule[1pt]

\multirow{6}{*}{CIFAR100}&CGAN (\ag{})&83.09&	33.40&	38.15&	23.78&	17.48&& 256.37&	230.41&	207.10&	191.54&	191.35& \\
&CGAN (\pg{})& 99.54&	42.28&	24.91&	22.22&	13.90&&
256.61&	212.70&	\bfseries 196.53&	\bfseries 175.59&	\bfseries 170.55&\\
&CGAN (\ctg{})& 74.34&	47.89&	31.59&	17.19&	\bfseries 13.45&&
256.69&	213.73&	216.42&	200.58&	208.18&\\
&\ftg{}& 25.87&	\bfseries 18.43 &	\bfseries 17.64 &	\bfseries 15.57 &	14.45&&
249.77&	249.72&	237.16&	239.64&	236.70&\\
\cline{2-14}
&\rg{} (ours)&\bfseries 20.22 &	\bfseries 18.42 &	17.87&	17.04&	16.92&&
\bfseries 238.24 &	\bfseries 207.18 &	202.28&	201.19& 187.09&\\

\toprule[1pt]

\end{tabular}
}
\end{table*}

\begin{table*}[t]
\vspace{3mm}
\renewrobustcmd{\bfseries}{\fontseries{b}\selectfont}
    \sisetup{detect-weight,mode=text,group-minimum-digits = 4}
\caption{\textbf{Recall and \cas{} evaluations on \dc{}.} 
We measure \gancond{} performance with recall and \cas{} varying amounts of labeled data ($\rx$\% on \dc{} means that $\rx$\% of labeled \dc{} is randomly selected).
\rg{} outperforms baselines in terms of recall and \cas{} when small amounts of labeled data (1\%, 10\%, 20\%) are available. 
The performance gap between \rg{} and others becomes larger as less supervision is provided.
When more labeled data is provided (50\%, 100\%), \rg{} performs competitively to the best models in terms of recall and still outperforms others in terms of \cas{}.}
    
\label{tab:recall_cas_cifar10}
    \centering
\begin{tabular}{cSSSSSSSSSS}
\toprule[1pt]
\multirow{2}{*}{Method}&
\multicolumn{5}{c|}{Recall~($\uparrow$)}&
\multicolumn{5}{c}{CAS~($\uparrow$)}\\
& {1\%} & {10\%} & {20\%} & {50\%} & \multicolumn{1}{c|}{100\%} &{1\%} & {10\%} & {20\%} & {50\%} & {100\%}\\
\toprule[1pt]
CGAN (\ag{})
& 0.12  &  0.59  &  0.66  &  0.72  &  0.78  
& 11.63  &  29.83  &  33.04  &  39.70  &  43.99 \\
CGAN (\pg{})
& 0.12  &  0.56  &  0.67  &  0.72  &  \bfseries 0.80  
& 21.52  &  28.09  &  30.03  &  32.43  &  35.06 \\
CGAN (\ctg{})
& 0.23  &  0.59  &  0.67  &  \bfseries 0.79  &  0.77  & 12.00  &  14.50  &  15.02  &  21.91  &  18.78 \\
\ftg{}
& 0.64  &  0.76  &  0.77  &  0.77  &  0.79  & 11.50  &  12.68  &  15.75  &  15.56  &  15.66 \\
\grep{}
& 0.65  &  0.69  &  0.72  &  0.75  &  0.76  & 21.08  &  23.95  &  26.94  &  27.56  &  29.64 \\
\cline{1-11}
\rg{} (ours)
& \bfseries 0.72  &  \bfseries 0.77  &  \bfseries 0.78  &  0.78  &  0.79  & \bfseries 22.03  &  \bfseries 31.37  &  \bfseries 35.19  &  \bfseries 47.80 &  \bfseries 46.38 \\

\toprule[1pt]
\end{tabular}

\end{table*}
\vspace{-3mm}
We compare \gancond{} methods under different amounts of labeled data.
Specifically, we use different proportions of the training dataset (1\%, 10\%, 20\%, 50\%, and 100\%) as the labeled data.
Table~\ref{tab:fid_supervision} shows the comparison in terms of \fid{} and \cfid{} on \dc{} and \dhc{} datasets.
\td{Table~\ref{tab:recall_cas_cifar10}} provides additional measures (recall and \cas{} scores) of the methods on \dc{}.
We also compare the quality of conditional samples on \dc{} in Fig.~\ref{fig:cifar10_visual_comparison}.

\paragraph{\rg{} outperforms baselines in the regime of low supervision.}
\textit{When the supervision level is $10$\% or less}, \rg{} consistently achieves the best scores (\fid{} and \cfid{}) over all datasets.
For instance, \fid{} scores of \rg{} at 1\%  are $16.39$ and $20.22$ for \dc{} and \dhc{} (Table~\ref{tab:fid_supervision}).
The corresponding \cfid{} scores of \rg{} are $76.24$ and $238.24$, showing large gaps ($38.27$ and $11.53$) to the second-best models. 
Notably, these score gaps between \rg{} and other models become larger as less supervision is provided.
Regarding recall and CAS metrics on \dc{}, Table~\ref{tab:recall_cas_cifar10} also supports our findings: \rg{} outperforms baselines in both recall and \cas{} when small amounts of  labeled data are available.

We attribute good performances of \rg{} to its design that utilizes the representation of trained generators and efficiently uses limited labels in separating conditional latent vectors (Sec.~\ref{sec:inrep_model}).
Interestingly, thanks to well-trained generators, \grep{} and fine-tuning also achieve better performances than other \cg{}s.
For the qualitative comparison, Fig.~\ref{fig:cifar10_visual_comparison} partly illustrates that most samples from \rg{} have higher image quality than those from \cg{} methods and are more correctly conditioned than samples from fine-tuning and \grep{}.

\textit{As more labeled data is provided (20\% and more)}, \cg{}s gradually obtain better scores and outperform \rg{}, but \rg{} manages to keep relatively small gaps to the best \cg{}s.
For instance, with full supervision on \dc{}, these gaps between \rg{} and the best model, \pg{}, are just $1.1$ and $7.5$ in terms of \fid{} and \cfid{}.

\subsection{Robustness to class-imbalanced data and noisy supervision}
\label{sec:exp_robustness}
We analyze two practical settings of the labeled data, when the data is class-imbalanced and when the labels are noisy.

\begin{table*}[t]
    \renewrobustcmd{\bfseries}{\fontseries{b}\selectfont}
    \sisetup{detect-weight,mode=text,group-minimum-digits = 4}
    \vspace{-3mm}
    \caption{
    \textbf{\fid{} evaluation on class-imbalanced \dc{}}. The dataset has 10 classes where classes $0$ and $1$ are minor.
    Compared to the \fid{} on class-balanced data (Table~\ref{tab:fid_supervision}),
    \rg{} and \grep{} obtain the smallest \fid{} increases among methods, indicating that they are less affected by the imbalance in labels.
    Also, \rg{} achieves the best \cfid{}s for minor classes, close to the average score of major classes, while the \cfid{} gaps between the minor and major classes are larger in other models.}
    \centering
    {
    \begin{tabular}{cSSSS}
        \hline
        \toprule[1pt]
        \multirow{2}{*}{Method} & {\multirow{2}{*}{FID ($\downarrow$)}} & \multicolumn{3}{c}{\cfid{} ($\downarrow$)}\\
        \cline{3-5}
        & & \text{Class 0} & \text{Class 1} & \text{Major classes}\\
        \hline
        CGAN (\ag{}) & 19.87 & 124.25 & 141.89 & 89.29\\
        CGAN (\pg{}) & 14.80& 75.87  & 80.59       & \bfseries 53.13     \\
        CGAN (\ctg{}) & 18.95       &   139.75     &   84.20 &  135.08  \\
        \ftg{} & 14.43 & 117.52&     100.03  & 89.29   \\
        \grep{} & 16.01&   73.06     &  88.92      &   68.59   \\
        \hline
        \rg{} (ours) & \bfseries 13.26 &  \bfseries 69.48 & \bfseries 61.68 & 60.06\\
        \hline
        \toprule[1pt]
    \end{tabular}
    }
    \label{tab:fid-imbalance}
\end{table*}

\paragraph{\rg{} outperforms \gancond{} baselines on class-imbalanced \dc{}}
The class-imbalanced \dc{} has two minor classes ($0$ and $1$) that use 10\% of their labeled data, while other classes remain unchanged.
Training on this augmented dataset, we observe that \rg{} obtains the lowest overall \fid{} ($13.26$), shown in Table~\ref{tab:fid-imbalance}.
\rg{} also achieves the best \cfid{} scores for both two minor classes ($69.48$ and $61.68$), which are close to the average \cfid{} of major classes ($60.06$).
On the other hand, other methods show larger gaps between \cfid{} of minor classes and the average \cfid{} of major classes.
For instance, these score gaps in \pg{} and \ag{} are all greater than $20$.
This result validates the advantage of \rg{} in learning with class-imbalanced labeled data.

\begin{remark}[Fairness in the data generation]
The under-representation of minority classes in training data may bias models towards generating samples with more features of majority classes or samples of minority classes with lower quality and less diversity~\citep{kenfack2021fairness,mehrabi2021survey,ferrari2021addressing}.
Being more robust against the class imbalance, \rg{} is a promising candidate for \gancond{} with fair generation.
\end{remark}

\paragraph{\rg{} is more robust to label noises in \dc{}}

\begin{table*}[h]
\vspace{3mm}
    \renewrobustcmd{\bfseries}{\fontseries{b}\selectfont}
    \sisetup{detect-weight,mode=text,group-minimum-digits =4}
    \caption{\textbf{\fid{} evaluation on label-noisy \dc{}.}
    The label-noisy \dc{} is constructed by corrupting labels of \dc{} with asymmetric noises.
    Compared to reported scores on the label-clean data (Table~\ref{tab:fid_supervision}), \rg{} and \grep{} maintain small \fid{}s (13.15 and 15.84) and \cfid{}s (63.15 and 71.51) under the label corruption, while other methods suffer from high increases.
    For instance, the increased \cfid{} gaps of \ag{} and \pg{} are all greater than 30.
    Furthermore, \rg{} and \grep{} also maintain the smallest \cfid{}s for both noisy and clean classes.
    These results indicate that \irep{} approaches are more robust against label noises. 
    }
    \vspace{0.1em}
    \centering
        \begin{tabular}{cSSSS}
        \hline
        \toprule[1pt]
        \multirow{2}{*}{Method} & {\multirow{2}{*}{FID ($\downarrow$)}} & \multicolumn{3}{c}{\cfid{} ($\downarrow$)}\\
        \cline{3-5}
        & & \text{All classes} & \text{Noisy classes} & \text{Clean classes}\\
        \hline
        CGAN (\ag{})  & 27.63 & 129.11                      & 140.64 & 121.42    \\
        CGAN (\pg{})  & 15.44 & 86.85                     & 91.30& 83.90 \\
        CGAN (\ctg{}) & 14.10 & 149.49                         & 96.49& 184.83 \\
        \ftg{} &  16.18      &   115.91& 125.52    &   106.00 \\
        \grep{} &  15.84      &   71.51 &  73.10 & 69.92
        \\ \hline
        \rg{} (ours)  & \bfseries 13.15 & \bfseries 63.15 & \bfseries 65.93 & \bfseries 61.30 \\ \hline
        \toprule[1pt]
        \end{tabular}
    \label{tab:cfid-affected-cifar10}
\end{table*}

We adopt the class-dependent noise setting in robust image classification~\citep{thekumparampil2018robustness,kim2017learning}, and corrupt clean labels with flipping probability $0.4$ as follows: bird $\to$ airplane, cat $\to$ dog, deer $\to$ horse, truck $\to$ automobile.
Shown in Table~\ref{tab:cfid-affected-cifar10} are evaluated \fid{} and \cfid{}.
Though label corruption causes \fid{} increases in all models, the increase of \rg{} is only $0.16$ (compared to clean-label \fid{} in Table~\ref{tab:fid_supervision}), comparable to \grep{} and stark contrast with large gaps of \cg{} methods.
In terms of \cfid{}, \rg{} achieves the smallest score ($63.15$), and enjoys the smallest \cfid{} increases together with \grep{} ($3.5$ and $3.2$) among evaluated methods.
Also, their clean classes are less affected by label noises, while we observe non-negligible deviations in other models.
These results show the strong robustness of \irep{} methods against label noises.
We attribute this robustness to their modular design, which preserves the representation learned by unconditional generators and focuses on learning the conditional distribution.
Also, the multi-head design of the conditional discriminator contributes to preserving the separation between data of different classes.
Thus, corrupted labels only significantly affect the corresponding classes in \rg{} (and \grep{}) while affecting all classes in the joint training of other methods (CGANs and fine-tuning).

\subsection{Ablation study}
\label{sec:exp_ablation}

\begin{wraptable}{r}{0.47\textwidth}
\vspace{-3mm}
    \caption{\textbf{Analyzing the effect of \rg{}'s components.}
    Each row compares \rg{} (w/ PU-loss) with its version without PU-based loss (w/ PU-loss).
    Each column compares \rg{} (w/ invertibility) with its version when the modifier network is not invertible (w/o invertibility).
    The scores are \cfid{} on \dc{}.
    Both two components help improve \rg{}.
    Enabling invertibility reduces \cfid{} by nearly $11.7$ and $15.1$, while enabling the PU-based loss improves approximately $2.9$ and $6.3$.
    These results indicate the stronger effect of the invertibility than the PU-based loss on the final performance of \rg{}.
    }
    \centering
        \begin{tabular}{ccc}
        \toprule[1pt]
        Setting &  w/ PU-loss & w/o PU-loss \\ 
        \hline
        w/ invertibility  & 64.25 & 67.14\\
        w/o invertibility & 75.92 & 82.26  \\      
        \toprule[1pt]
        \end{tabular}
    \label{tab:ablation}
    \vspace{-3mm}
\end{wraptable}

We investigate the empirical effect of the invertible architecture and the PU-based loss on \rg{}.
We construct four \rg{} versions from enabling or disabling each of two components.
We use ResNet architecture~\citep{he2015deep} and the standard GAN loss as replacements for the i-ResNet architecture and our PU-based loss.
UGAN is pretrained on unlabeled \dc{}, and the labeled data is $10\%$ of \dc{} (to observe better the effect of PU-based loss under the low-label regime).
Shown in Table~\ref{tab:ablation} are the \cfid{} scores of four models. 
The original \rg{} achieves the best score ($64.25$), while removing all components causes a higher \cfid{} ($82.26$).
Using PU-based loss and invertibility help \rg{} reduce \cfid{}, with the biggest decreases being $6.3$ and $15.1$, respectively.
Thus both components have positive effects on the \rg{}'s performance, and invertibility shows a higher impact on \rg{}.

\section{Discussion}
\label{sec:discussion}

\paragraph{\td{Limitation}} \rg{} depends greatly on the pretrained \ugen{}: how perfectly the \ugen{} learns the underlying data distribution and how well its latent space organizes. 
The latest research~\citep{lucic2019high,liu2020diverse} observes the existing capacity gaps between unconditional \gan{}s and conditional \gan{}s on large-scale datasets.
Therefore, how to improve \rg{} on large-scale datasets is still a challenging problem, especially when \gan{}s' latent space structure has still not been well-understood.
We leave this open problem for future work.

\paragraph{Beyond \gan{}: Can we apply \rg{} to other generative models?}
Our work currently focuses on \gan{}s for studying \gancond{} because \gan{}s attain state-of-the-art performances in various fields~\cite{lucic2019high,kang2020contragan,isola2017image,zhu2017unpaired,yu2017seqgan,guo2018long}, and numerous well-trained \gan{} models are publicly available.
Nevertheless, \rg{} \td{can be} applied to other unconditional generative models.
We can replace the \ugen{} from UGAN with one from other generative models, such as Normalizing Flows (Glow~\citep{kingma2018glow}) or Variational Autoencoders (VAEs)~\citep{Kingma_2019}.
Interestingly, we notice that GAN training (as of \rg{}) using generators pretrained on VAEs probably helps stabilize the training, reducing the mode collapse issue~\citep{ham2020unbalanced}.

\paragraph{Enhancing \irep{} to prompt tuning} 
Text prompts, which are textual descriptions of downstream tasks and target examples, 
have been shown to be effective at conditioning the GPT-3 model~\citep{brown2020language}. 
Multiple works have been proposed to design and adapt prompts for different linguistic tasks, such as prompt tuning~\citep{lester2021power} and prefix tuning~\citep{li2021prefix}.
The latter approach freezes the generative model and learns the prefix activations in the encoder stack.
Sharing similar merit to prompt tuning, \rg{} can be a promising solution for learning better prompts in conditioning generative tasks while achieving significant computing savings by freezing the generative model.

\section{Conclusion}
\label{sec:conclusion}
\td{
In this study, we define the \textit{\gancond{}} problem and thoroughly review three existing approaches to this problem.
We focus on \irep{}, the best performing approach under the regime of scarce labeled data.
We propose \rg{} as a new algorithm for improving the existing \irep{} algorithm over critical issues identified in our analysis.
In \rg{}, we adopt the invertible architecture for the modifier network and a multi-head design for the discriminator.
We derive a new GAN loss based on the Positive-Unlabeled learning to stabilize the training.
\rg{}
exhibits remarkable advantages over the existing baselines: The ideal \rg{} training provably learns true conditional distributions with perfect unconditional generators, and 
\rg{} empirically outperforms others when the labeled data is scarce, as well as bringing more robustness against the label noise and the class imbalance.}

\bibliography{refs}
\bibliographystyle{unsrt}

\newpage
\appendix

In (\ref{sec:app_proof}), we present \td{full statements}, proofs, and examples for all propositions, lemmas, and theorems in the main paper. 

In (\ref{sec:app_experiment}), we present the details of the experiments' setting, including architectures, datasets, and evaluation metrics.

\setcounter{theorem}{0}
\setcounter{proposition}{1}
\setcounter{lemma}{0}

\section{\td{Full Statements, Proofs, and Examples}}
\label{sec:app_proof}

\subsection{\td{Full statements and proofs} for Section~\ref{sec:prelim}}
\label{app:proof_prelim}
\subsubsection{Examples of the failure of \ag{}}

\ag{} has two objective terms $L_S, L_C$, modeling log-likelihoods of samples being from real data (by unconditional discriminator $D: \sX \mapsto \{\real, \fake\}$) and being from correct classes (by auxiliary classifier $D_{\textrm{aux}}: \sX \mapsto \sY$), respectively. 
Then, the discriminators $(D, D_{\textrm{aux}})$ and generator $G$ maximize $\lambda L_C + L_S$, $\lambda L_C - L_S$ respectively with a regularization coefficient $\lambda$. 
As in the vanilla \gan{}s, $L_S$ term encourages $G, D$ to learn the true (\td{that is}, unbiased) distribution, while $L_C$ term prefers easy-to-classify (\td{that is}, possibly biased) distributions.
Therefore, $G$ might be able to maximize $\lambda L_C - L_S$ by learning a biased distribution (increased $L_C$) at the cost of compromised generation quality (increased $L_S$).
Indeed, this phenomenon was observed in~\citep{shu2017ac}, where the authors provided theoretical reasons and empirical evidence that \ag{}-generated data are biased.
We corroborate their finding by providing two examples in which \ag{} provably learns a biased distribution. For simplicity, we consider the Wasserstein \gan{} version~\citep{martin2017wasserstein} of \ag{} (W-\gan{}).%

\begin{lemma}[W-ACGAN fails to learn a correct Gaussian mixture, full statement]
\label{lemma:acgan_1}
Consider $1$-dimensional real data $\rx \sim \mathcal{N}(y, 1)$ where $y = 2\Bernoulli(1/2)-1$. 
Assume the perfect discriminator and the following generator, parameterized by $v$: $G(\rvz, y) \sim \mathcal{N}(y, v^2), v \in \mathbb{R}_+$ for given label $y$, \td{that is}, the only model parameter is \td{the standard deviation $v$}. 
Let $v^*(\lambda)$ be the optimal generator's parameter that maximizes $
\lambda L_C - L_S$.
Then, $v^*(\lambda) \to 0$ as $\lambda \to \infty$.
\end{lemma}

\begin{proof}
Recall that $L(v, \lambda)= \lambda L_C - L_S$ for the generator. From the fact that the $2$-Wasserstein distance between two Gaussian random variables is
\begin{align*}
    W_2(\mathcal{N}(m_1, \sigma_1^2), \mathcal{N}(m_2, \sigma_2^2)) = \sqrt{(m_1 -m_2)^2 + (\sigma_1 - \sigma_2)^2},
\end{align*}
the term $L_S$ reduces to a simple expression $\E_{\ry}[ W_2(\rx, G(\rvz, \ry)) ] = |1-v|$. Therefore, when there is no auxiliary discriminator, maximizing $L(v, \lambda)$ is indeed the same as minimizing $L_S$. Then, the \ag{} will find the global minimum $v^*=1$ via gradient descent as $L_S$ is convex. We will show that this is no longer true if we consider the auxiliary classification term $L_C$.

Consider the $L_C$ term that also has a closed-form expression:
\begin{small}
\begin{align*}
&\E_{\rx}\left[\1_{\{D_{\textrm{aux}}(\rx) = \ry \}}\right] + \mathbb{E}_{\rvz, \ry}\left[\1_{\{D_{\textrm{aux}}(G(\rvz, \ry)) = \ry \}}\right] \\
&= 2-\left( \frac{1}{2} \int_{0}^{\infty} p(x | y=-1) \dif x + \frac{1}{2} \int_{-\infty}^{0} p(x |y=1) \dif x \right)  - \left( \frac{1}{2} \int_{0}^{\infty} p(G(\vz, y)|y=-1) \dif x + \frac{1}{2} \int_{-\infty}^{0} p(G(\vz, y)|y=1) \dif x \right) \\
&= 2-\frac{1}{2} Q(1) - \frac{1}{2} Q(1) - \frac{1}{2} Q(v^{-1}) - \frac{1}{2} Q(v^{-1}) = 2-Q(1) - Q(v^{-1}),
\end{align*}
\end{small}
where $Q(\cdot)$ is the complementary cumulative distribution function of the standard Gaussian, \td{that is}, $Q(t) = \int_{t}^{\infty} \mathcal{N}(\tau;0,1) d\tau$.

Therefore, the overall objective of the \ag{} to be minimized by $G$ is
\begin{align*}
L(v, \lambda) = |1-v| + \lambda \left( Q(1) + Q(v^{-1}) \right) + const.
\end{align*}

To see the behavior of the best generator, \td{that is}, $v^*=v^*(\lambda)$, first assume $v < 1$. Using the chain rule and the property of $Q$ function that $\tfrac{d}{\dif t} Q(t) = - \tfrac{1}{\sqrt{2\pi}} \exp(-t^2/2)$,
\begin{align*}
\frac{\partial L}{\partial v} = -1 + \frac{\lambda}{\sqrt{2\pi}} \exp\left( -\frac{1}{2v^2} \right) \frac{1}{v^2} ~~~ \textrm{for } v < 1.
\end{align*}
As $\exp(-\tfrac{1}{2v^2}) \tfrac{1}{v^2}$ is positive on $(0,1)$, the derivative is positive on $(0,1)$ as well by taking a large $\lambda$. Therefore, the generator will learn $v^*(\lambda) \to 0$ as $\lambda \to \infty$. When $v \ge 1$, 
\begin{align*}
\frac{\partial L}{\partial v} = 1 + \frac{\lambda}{\sqrt{2\pi}} \exp \left( -\frac{1}{2v^2} \right) \frac{1}{v^2} > 1 ~~ \text{for all } v.
\end{align*}
So, the global minimum $v^*(\lambda)$ of $L(v, \lambda)$ converges to $0$ as $\lambda$ tends to infinity. 
Fig.~\ref{fig:ex1_fig} shows the contour of $L(v,\lambda)$ and demonstrates  that $v^*(\lambda)$ converges to $0$ as $\lambda \to \infty$.
\end{proof}

\begin{figure*}[t]
	\begin{subfigure}{.33\textwidth}
		\centering
		\includegraphics[height=2in]{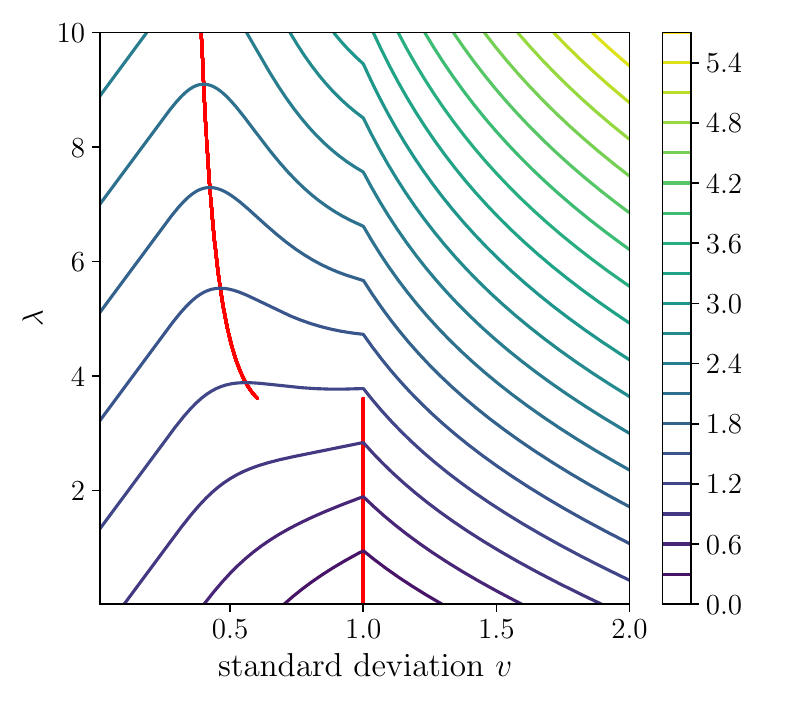}
		\caption{}
		\label{fig:ex1_fig}
	\end{subfigure} ~
	\begin{subfigure}{.33\textwidth}
	\centering
	\includegraphics[height=2in]{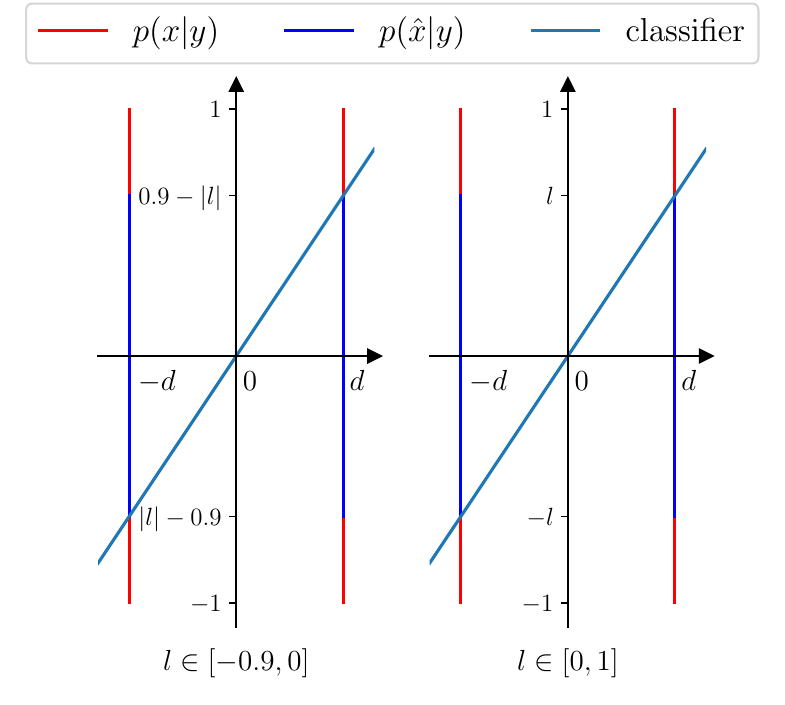}
	\caption{}
	\label{fig:ex2_setting}
	\end{subfigure} ~
	\begin{subfigure}{.33\textwidth}
		\centering
		\includegraphics[height=2in]{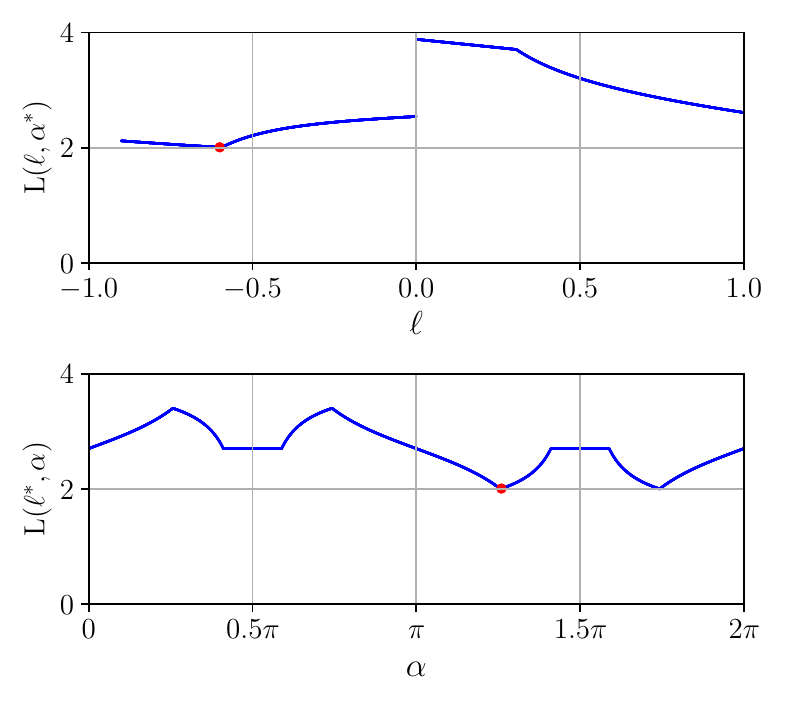}
		\caption{}
		\label{fig:ex2_fig}
	\end{subfigure}
	\vspace{-2mm}
	\caption{
	\textbf{Visual illustrations for Lemma~\ref{ex:non_separable} and Lemma~\ref{ex:separable}.}
	(a) Loss contour plot for $L(v, \lambda)$ in Lemma~\ref{ex:non_separable}. 
	Minimizer $v^*(\lambda)$ is marked in red.
	\td{We can see that the global minimum $v^*(\lambda)$ of $L(v, \lambda)$ converges to $0$ as $\lambda$ tends to infinity.} 
	(b) Setting for Lemma~\ref{ex:separable}.
	\td{Each sub-figure illustrates the target distribution (red), the generated distribution (blue) and the auxiliary classifier (pine).
	Sub-figures (from left to right) are for $\ell \in [-0.9, 1]$ and $\ell \in [0, 1]$,  respectively.}
	(c) Cross section plots for $L(\ell, \alpha, \lambda)$ in Lemma~\ref{ex:separable} at $\ell^*=-0.6, \alpha^*=\pi + \tan^{-1}( \tfrac{0.3}{d} ) \approx 1.26 \pi, \lambda^*=2$, which is a bad critical point. }
\end{figure*}
\begin{lemma}[W-ACGAN with gradient descent fails to learn a separable distribution, full statement]
Consider real and fake data that are vertically uniform in $2$-dimensional space. 
Real data conditioned on $y \in \{\pm 1\}$ are located at $\rvx=(\rx_1, \rx_2)=(d \cdot \ry, \ru)$ where $d= \sqrt{0.99/12}$ and $\ru \sim \Unif[-1,1]$.
Assume that the generator creates fake data from a distribution parameterized by $\ell \in [-0.9, 1]$: fake data $G(\rvz,y)=(\hat{\rx}_1, \hat{\rx}_2)$ where $\hat{\rx}_1 = d \cdot y, \hat{\rx}_2 \sim \Unif[-\ell, \ell]$ if $\ell  \in [0, 1]$ and $\hat{\rx}_1=-d \cdot y, \hat{\rx}_2 \sim \Unif[-(0.9-|\ell|), 0.9-|\ell|]$ otherwise. 
An auxiliary classifier is linear that passes the origin such that $\hat{y}(\vx) = \sgn(\tan(\alpha) x_1 - x_2), \alpha \in [0, 2\pi)$. 
The loss terms are:
\begin{align*}
    L_S &= \E_{\ry} \left[ W_2(  \rvx, G(\rvz, \ry))  ) \right] \\
    L_C &= \E_{\rvx}\left[\1_{\{D_{\textrm{aux}}(\rvx) \ne \ry \}}\right] + \mathbb{E}_{\rvz, \ry}\left[\1_{\{D_{\textrm{aux}}(G(\rvz, \ry)) \ne \ry \}}\right],
\end{align*}
Note that $L_C$ term indicates classification \emph{error} instead of accuracy. 
Let $L(\ell, \alpha, \lambda)$ be $L_S + \lambda L_C$.
Assume the perfect discriminator.
In this setting, $L(\ell, \alpha, \lambda)$ has bad critical points, so gradient descent-based training algorithms may not converge to the global optimum.
\end{lemma}

\begin{proof}
Recall the Wasserstein distance between two $1$-dimensional measures:
\begin{align*}
W_2(\rz_1, \rz_2) = \left( \int_{0}^{1} |F_{\rz_1}^{-1}(t) - F_{\rz_2}^{-1}(t)|^2 \dif t \right)^{1/2},
\end{align*}
where $F^{-1}$ is the inverse of the cumulative distribution function of a random variable. First suppose $0 \le \ell \le 1$ so that the supports of $p_{\rvx | \ry}(\vx|y), p_{G(\rvz, \ry)}(\vx|y)$ overlap. As our real and fake distributions are uniform,
\begin{align*}
L_S &= \E_{\ry} \left[ W_2(\rvx, G(\rvz, \ry)) \right] = W_2 \left( \Unif[-1, 1], \Unif[-\ell, \ell] \right) \\
&= \left( \int_{0}^{1} |(2t-1) - (2t \ell -\ell)|^2 \dif t \right)^{1/2} = \frac{1}{\sqrt{3}} (1-\ell) \quad \textrm{when } 0 \le \ell \le 1.
\end{align*}
Suppose $-0.9 \le \ell < 0$. Then, by the translation property of the Wasserstein distance, e.g., Remark 2.19 in~\citep{peyre2019computational},
\begin{align*}
\left( \E_{\ry} \left[ W_2(\rx, G(\rvz,\ry)) \right] \right)^2 &= (2d)^2 + W_2 \left(
\Unif[-1, 1], \Unif[-(0.9-|\ell|), 0.9-|\ell|] \right)^2 \\
&= 4d^2 + \frac{(0.1-\ell)^2}{3} \quad \textrm{when } -0.9 \le \ell < 0.
\end{align*}
Therefore, the Wasserstein distance is given by
\begin{align*}
L_S = \E_{\ry} \left[ W_2(\rx, G(\rvz,\ry)) \right] = \sqrt{4d^2 + \frac{(0.1-\ell)^2}{3}} \1_{\{ -0.9 \le \ell < 0 \}} + \frac{1-\ell}{\sqrt{3}} \1_{\{ 0 \le \ell \le 1\}}.
\end{align*}
The Wasserstein distance monotonically decreases in $\ell$, so the vanilla (Wasserstein) CGAN finds $\ell^*=1$ by gradient descent. In this example, $d$ is chosen so that $L_S$ is continuous, and the mismatched range of $\ell$ is selected so that the errors of the auxiliary classifier on true data and fake data do not symmetrically move. We will see below that this solution could not be learnable by the \ag{} with gradient descent if one begins from a bad initial point.

Notice that a closed-form expression for $\E_{\rvx}\left[\1_{\{ D_{\textrm{aux}}(\rvx) \ne \ry \}}\right]$ in terms of $\alpha$ is given as follows. Letting $\alpha_0 = \tan^{-1}(d^{-1})$,
\begin{align*}
\E_{\rvx}\left[\1_{\{ D_{\textrm{aux}}(\rvx) \ne \ry \}}\right] = \begin{cases}
\frac{1}{2} - \frac{d\tan(\alpha)}{2} & \textrm{if } \alpha \in [2\pi -\alpha_0, 2\pi) \cup [0, \alpha_0] \\
0 & \textrm{if } \alpha \in (\alpha_0, \pi - \alpha_0) \\
\frac{1}{2} + \frac{d\tan(\alpha)}{2} & \textrm{if } \alpha \in [\pi - \alpha_0, \pi + \alpha_0] \\
1 & \textrm{if } \alpha \in (\pi + \alpha_0, 2\pi-\alpha_0).
\end{cases}
\end{align*}
Let $\alpha_+(\ell) \coloneqq \tan^{-1}(\tfrac{\ell}{d})$ and $\alpha_-(\ell) \coloneqq \tan^{-1}(\tfrac{0.9-|\ell|}{d})$, 
we can similarly obtain a closed-form expression for $\mathbb{E}_{p_{\rvz}(\rvz)p_{y}(y)}\left[\1_{\{D_{\textrm{aux}}(G(\rvz, y)) \ne y \}}\right]$.
\begin{align*}
\textrm{If } \ell \ge 0, \quad \mathbb{E}_{\rvz, \ry}\left[\1_{\{D_{\textrm{aux}}(G(\rvz, \ry)) \ne \ry \}}\right] = \begin{cases}
\frac{1}{2} - \frac{d \tan(\alpha)}{2\ell} & \textrm{if } \alpha \in [2\pi-\alpha_+, 2\pi) \cup [0, \alpha_+] \\
0 & \textrm{if } \alpha \in (\alpha_+, \pi - \alpha_+) \\
\frac{1}{2} + \frac{d \tan(\alpha)}{2\ell} & \textrm{if } \alpha \in [\pi-\alpha_+, \pi+\alpha_+] \\
1 & \textrm{if } \alpha \in (\pi+\alpha_+, 2\pi-\alpha_+).
\end{cases}
\end{align*}
\begin{align*}
\textrm{If } \ell < 0, \quad \mathbb{E}_{\rvz, \ry}\left[\1_{\{D_{\textrm{aux}}(G(\rvz, \ry)) \ne \ry \}}\right] = \begin{cases}
\frac{1}{2} + \frac{d \tan(\alpha)}{2(0.9-|\ell|)} & \textrm{if } \alpha \in [2\pi-\alpha_-, 2\pi) \cup [0, \alpha_-] \\
1 & \textrm{if } \alpha \in (\alpha_-, \pi - \alpha_-) \\
\frac{1}{2} - \frac{d \tan(\alpha)}{2(0.9-|\ell|)} & \textrm{if } \alpha \in [\pi-\alpha_-, \pi+\alpha_-] \\
0 & \textrm{if } \alpha \in (\pi+\alpha_-, 2\pi-\alpha_-).
\end{cases}
\end{align*}

Let us investigate a generator obtained via alternating optimization. Fix an arbitrary $\ell^* < 0$. Note that at $\alpha \in (\pi, \pi + \alpha_-(\ell^*))$, $\mathbb{E}_{\rvx}[\cdot]$ is increasing, while $\mathbb{E}_{\rvz, \ry}[\cdot]$ is decreasing faster. So $L_C$ is decreasing. However, as $\pi + \alpha_- < \pi + \alpha_0$, at $\alpha \in (\pi + \alpha_-(\ell^*), \pi+\alpha_0)$, $\mathbb{E}_{\rvx}[\cdot]$ is increasing while $\mathbb{E}_{\rvz, \ry}[\cdot]$ is constant. So $L_C$ is increasing. Hence, $L_C$ attains its local minimum $\alpha^*=\pi+\alpha_-(\ell^*)$ by gradient descent.

Now fix $\alpha^* = \pi + \alpha_-(\ell^*)$. As $\mathbb{E}_{\rvz, \ry}[\cdot]$ is decreasing on $\ell < \ell^*$ and increasing on $\ell > \ell^*$, while $\mathbb{E}_{\rvx}[\cdot]$ remains unchanged. So $L_C$ term has its local minimum at $\ell^*$ for a given $\alpha^*$. Note that $L_S$ is decreasing in $\ell$. Hence by picking large $\lambda^*$ such that $L_C$ term dominates $L_S$, the generator learns $\ell^* \ne 1$ by gradient method. Fig.~\ref{fig:ex2_fig} depicts the cross-section of $L(\ell,\alpha,\lambda)$ at $\ell^*=-0.6$, $\alpha^* =\pi + \tan(\tfrac{0.3}{d}) \approx 1.26\pi$, $\lambda^*=2$ and demonstrates this is a bad critical point.
\end{proof}

\subsubsection{Examples of the failure of \pg{}}
We provide an example to show why and when \pg{}~\citep{miyato2018cgans} may fail to work. 
Note that the \pg{}'s discriminator takes an inner product between a feature vector and the class embedded vector, that is,
\begin{align*}
    D(\vx, \vy; \vtheta_D) = \vy^{T}\mV\phi(\vx) + \psi(\phi(\vx)),
\end{align*}
where $\vy$ is an one-hot encoded label vector, $\mV=[\vv_1^T;\vv_2^T;\cdots]$ is the label embedding matrix of $\rvy$, and $\phi$, $\psi$ are learnable functions.
This special algebraic operation reduces to expectation matching in a simple setting, by which we can provably show that \pg{} mislearns the exact conditional distributions.

\begin{lemma}[full statement]
Assume $\phi(\vx) = \vx$ and $\psi(\vx) = 0$ for an equiprobable two-class data. Thus, $\vtheta_D = \{\mV\}$ with $\mV=
\begin{bmatrix}
\vv_{0}^{T}\\
\vv_{1}^{T}
\end{bmatrix}$. With hinge loss, the loss for the discriminator to minimize is written as
\begin{align*}
    L_D &= \E_{\rvx, \ry} \left[ \max(0, 1-D(\rvx, \ry; \vtheta_D)) \right] + \E_{\rvz, \ry} \left[ \max(0, 1+D(G(\rvz, \ry), \ry; \vtheta_D)) \right].
\end{align*}
Let us further \td{suppose that} the generator learned the exact conditional distributions at time $t^*$, \td{that is}, $G_{t^*}(\rvz,y) \sim \rvx|\ry=y$, where $G_t$ stands for the generator at time $t$. Then, there exist discriminator's bad embedding vectors that encourage the generator to deviate from the exact conditional distributions.
\end{lemma}

\begin{proof}
Note that the discriminator $D$ under the setting is
\begin{equation}
\begin{split}
    D(\vx, \vy; \vtheta_D) 
    &= \vy^{T}\mV\phi(\vx) + \psi(\phi(\vx))\\
    &= \vy^{T}\mV\vx.
\end{split}
\end{equation}
With hinge loss, the loss for the discriminator to minimize is written as
$$
    L_D = \E_{\rvx, \ry} \left[ \max(0, 1-D(\rvx, \ry; \vtheta_D)) \right] + \E_{\rvz, \ry} \left[ \max(0, 1+D(G(\rvz, \ry), \ry; \vtheta_D)) \right].
$$

For simple presentation, we restrict ranges of $\rvx, \vv_y$ properly so that $D = \vy^{T}\mV\vx$ is in $(-1,1)$. Then, we can omit the $\max$ operator. As classes are equally likely and our setting assumes $D = \vy^{T}\mV\vx$,
\begin{small}
\begin{align*}
    L_D &= \frac{1}{2}( \E_{p_{\text{data}}(\vx|y=0)} \left[ 1-D(\rvx, 0; \vtheta_D) \right] \\
    &+ \E_{p_{\text{data}}(\vx|y=1)} \left[ 1-D(\rvx, 1; \vtheta_D) \right] + \E_{\rvz, \ry=0} \left[ 1+D(G(\rvz, 0), 0; \vtheta_D) \right] + \E_{\rvz, \ry=1} \left[ 1+D(G(\rvz, 1), 1; \vtheta_D) \right]) \\
    &= \frac{1}{2} \E_{p_{\text{data}}(\vx|y=0)} \left[ 1- \vv_0^T\rvx \right] + \frac{1}{2} \E_{p_{\text{data}}(\vx|y=1)} \left[ 1-\vv_1^T\rvx \right] + \frac{1}{2} \E_{\rvz, \ry=0} \left[ 1+\vv_0^T G(\rvz, 0) \right] + \frac{1}{2} \E_{\rvz, \ry=1} \left[ 1+\vv_1^T G(\rvz, 1) \right].
\end{align*}
\end{small}

Hence, the gradient is in a simple form:
\begin{equation}
    \nabla_{\vv_y} L_D = -\frac{1}{2} \E_{p_{\text{data}}(\vx|y)} [\rvx] + \frac{1}{2} \E_{\rvz, \ry=y} [G(\rvz, y)], \label{eq:PG_LD}
\end{equation}
which means that the discriminator measures the difference between conditional means of true data and generated data.

Note that the loss for the generator to minimize is
\begin{align*}
    L_G &= -\E_{\rvz, \ry=0}[D(G(\rvz, 0), 0; \vtheta_D)] -\E_{\rvz, \ry=1}[D(G(\rvz, 1), 1; \vtheta_D)] 
    \\ &= -\E_{\rvz, \ry=0}[\vv_0^T G(\rvz, 0)] -\E_{\rvz, \ry=1}[\vv_1^T G(\rvz, 1)].
    \vspace{-5mm}
\end{align*}
Its gradient is also in a simple form:
\begin{align}
    \nabla_{G} L_G = -\vv_0 - \vv_1. \label{eq:PG_LG}
\end{align}

Now we are ready to discuss why \pg{} may fail to learn correct distributions. Let us \td{suppose that} the generator learned the exact conditional distributions at time $t^*$, \td{that is}, $G_{t^*}(\rvz,y) \sim \rvx|\ry=y$. This in turn implies $\E_{p_{\text{data}}(\vx|y)}[\rvx] = \E_{\rvz, \ry=y}[G_{t^*}(\rvz,y)]$ for $y=0,1$. Note that (\ref{eq:PG_LD}), (\ref{eq:PG_LG}) yield the following update:
\begin{align*}
    G_{t+1}(\rvz, y) &= G_{t}(\rvz, y) - \eta \nabla_{G} L_G = G_{t}(\rvz, y) + \alpha (\vv_{0,t} + \vv_{1,t}) \\
    \vv_{y,t+1} &= \vv_{y,t} - \eta \nabla_{\vv_y} L_D = \vv_{y,t} + \frac{\alpha}{2} \left( \E_{p_{\text{data}}(\vx|y)} [\rvx] - \E_{\rvz, \ry=y} [G_t(\rvz,y)]  \right).
\end{align*}
Therefore, assuming that the step size $\alpha$ is small, alternating update yields the following dynamics after $t^*$.

\textrm{First update: } 
\begin{align*}
G(\rvz, y) &\leftarrow G_{t^*}(\rvz, y) + \alpha (\vv_{0,t^*} + \vv_{1,t^*}) \\
\vv_y &\leftarrow \vv_{y,t^*} + \frac{\alpha}{2} \left( \E_{p_{\text{data}}(\vx|y)} [\rvx] - \E_{\rvz, \ry=y} [G(\rvz, y)]  \right) \\
&= \vv_{y,t^*} + \frac{\alpha}{2} \left( \E_{p_{\text{data}}(\vx|y)} [\rvx] - \E_{\rvz, \ry=y} \left[G_{t^*} + \alpha (\vv_{0,t^*} + \vv_{1,t^*})\right] \right) \\
&= \vv_{y,t^*} - \frac{\alpha^2}{2} (\vv_{0,t^*} + \vv_{1,t^*})  ~~~ \textrm{since the conditional first moments are the same}
\end{align*}
\textrm{Second update: }
\begin{align*}
G(\rvz, y) &\leftarrow G(\rvz, y) + \alpha (\vv_0 + \vv_1) \\
&=\left( G_{t^*}(\rvz, y) 
+ \alpha (\vv_{0,t^*} + \vv_{1,t^*}) \right) + \alpha \left( \vv_{0,t^*} - \frac{\alpha^2}{2} (\vv_{0,t^*} + \vv_{1,t^*}) + \vv_{1,t^*} - \frac{\alpha^2}{2} (\vv_{0,t^*} + \vv_{1,t^*}) \right) \\
&= G_{t^*}(\rvz, y) + 2\alpha (\vv_{0,t^*} + \vv_{1,t^*}) - \alpha^2 (\vv_{0,t^*} + \vv_{1,t^*}) \\
&\approx G_{t^*}(\rvz, y) + 2\alpha (\vv_{0,t^*} + \vv_{1,t^*}) ~~~\textrm{since $\alpha$ is small} \\
\vv_y &\leftarrow \vv_y + \frac{\alpha}{2} (-2\alpha (\vv_{0,t^*} + \vv_{1,t^*})) ~~~ \textrm{since the difference between first moments is zero at $t$} \\
&=\left( \vv_{y,t^*} - \frac{\alpha^2}{2} (\vv_{0,t^*} + \vv_{1,t^*}) \right) + \frac{\alpha}{2} \left( -2\alpha(\vv_{0,t^*} + \vv_{1,t^*}) \right) \\
&= \vv_{y,t^*} - \frac{3\alpha^2}{2}(\vv_{0,t^*} + \vv_{1,t^*})
\end{align*}
\textrm{Third update: }
\begin{align*}
G(\rvz, y) &\leftarrow G(\rvz, y) + \alpha (\vv_0 + \vv_1) \\
&=\left( G_{t^*}(\rvz, y) + 2\alpha (\vv_{0,t^*} + \vv_{1,t^*}) \right) + \alpha \left( \vv_{0,t^*} - \frac{3\alpha^2}{2}(\vv_{0,t^*} + \vv_{1,t^*}) + \vv_{1,t^*} - \frac{3\alpha^2}{2}(\vv_{0,t^*} + \vv_{1,t^*}) \right) \\
&\approx G_{t^*}(\rvz, y) + 3 \alpha (\vv_{0,t^*} + \vv_{1,t^*}) ~~~ \textrm{since $\alpha$ is small} \\
\vv_y &\leftarrow \vv_y + \frac{\alpha}{2}(-3 \alpha (\vv_{0,t^*} + \vv_{1,t^*})) \\
&= \vv_{y,t^*} - 3 \alpha^2 (\vv_{0,t^*} + \vv_{1,t^*}).
\end{align*}
Repeating the updates, the $G(\rvz, y)$ steadily diverges from $G_{t^*}$ unless $\vv_{0,t^*} + \vv_{1,t^*} = 0$. 
In other words, even when the generator exactly learns the true data at time $t^*$, nonvanishing $\nabla_{G}L_G$ may result in the divergence of the generator.
\end{proof}

\subsection{Full statements and proofs for Section~\ref{sec:inrep_revisit}}
\label{app:proof_31}
\td{We provide the following proposition to justify the use of reprogramming for \gancond{}.
This proposition implies that, given an unconditional generator learned on $p_{\text{data}}(\vx)$, one can take condition-specific noise $\rvz_y$ to generate data from the target conditional distribution $p_{\text{data}}(\vx|y)$.}
\begin{proposition}[full statement]
\label{prop:1}
Assume that for two random variables $\rvz$ and $\rvx$ having continuous probability density functions, we have a perfect continuous generator $G:\sZ \mapsto \sX$ that satisfies the following: (i) $G(\rvz) \eqd \rvx$, \td{(ii) $G^{-1}(\vx)$ is Lipschitz continuous around each $\vz_i \in G^{-1}(\vx)$, and (iii) the Jacobian of $G$ has finite operator norm}.
In addition, assume that $p_{\ry|\rvx}(y|\vx)$ is continuous in $\vx$ for all $y\in \sY$.
Then, for any discrete random variable $\ry$, possibly dependent on $\rvx$, we can construct a random variable $\rvz_{y}$ such that $G(\rvz_y) \eqd \rvx|\ry=y$.
\end{proposition}

\begin{proof}
Fix $y$ and take $\rvz_y$ as follows: 
\begin{align*}
    p_{\rvz_y}(\vz) = p_{\rvz}(\vz) \cdot \frac{p_{\ry|\rvx}(y|G(\vz))}{p_{\ry}(y)}.
\end{align*}
Let $B(\vx_0, r) \subset \sX$ be a small ball of radius $r$ centered at $\vx_0$. Since $\rvx$ has a continuous density function, we know that every $\vx$ is a Lebesgue point~\citep{bogachev2007measure}, which implies that
\begin{align}
    \frac{p_\rvx(B)}{\mu(B)} \to p_\rvx(\vx_0) ~~~ \textrm{as $r \to 0$}, \label{eq:Leb_eq1}
\end{align}
where $\mu(\cdot)$ is the Lebesgue measure.

On the other hand, as $G^{-1}$ could have finite (say at most $T$) multiple images, let
\begin{align*}
    G^{-1}(B) = \Omega_1 \cup \cdots \cup \Omega_t, ~~ t \le T,
\end{align*}
where $\Omega_i$ are disjoint each other.

With the above construction, we have the following.
\begin{align*}
    p_{\rvx}(G(\vz_y) \in B) &= p_{\rvz_y}(\vz_y \in G^{-1}(B)) = p_{\rvz_y} \left(\vz_y \in \sum_{i=1}^t \Omega_i \right) = \sum_{i=1}^t p_{\rvz_y}\left(\vz_y \in  \Omega_i\right).
\end{align*}

Then for each $i$, we can bound the probability as follows: Letting $z_0 = G^{-1}(\vx_0)$ in $\Omega_i$,
\begin{align}
    &\int_{\Omega_i} p_{\rvz_y}(\vz) \dif \vz = \int_{\Omega_i} p_{\rvz}(\vz) \frac{p_{\ry|\rvx}(y|G(\vz))}{p_{\ry}(y)} \dif \vz \nonumber \\
    &\stackrel{(a)}{=} \int_{\Omega_i} p_{\rvz}(\vz) \left( \frac{p_{\ry|\rvx}(y|\vx_0)}{p_{\ry}(y)} + h_i(\vz-\vz_0) \right)\dif \vz \nonumber \\
    &= p_{\rvz}(\Omega_i) \frac{p_{\ry|\rvx}(y|\vx_0)}{p_{\ry}(y)} + \int_{\Omega_i} p_{\rvz}(\vz) h_i(\vz-\vz_0) \dif \vz  \label{eq:identity_with_Taylor}
\end{align}
where (a) follows from the Taylor series with remainders being $h_i(\vz - \vz_0)$. 

Let us consider the second term. 
Because $G^{-1}$ is Lipschitz on $\Omega_i$, we can take a small ball $A(\vz_0, Kr)$ such that $\Omega_i \subset A(\vz_0, Kr)$ where $A(\vz_0, Kr)$ is the sphere centered at $\vz_0$ and of radius $Kr$ with the Lipschitz constant of $G^{-1}$ being $K$. 
Ignoring high-order terms of $h(\cdot)$ that are asymptotically negligible, the second term can be rewritten with the linear term only. 
Letting $\bm{g}$ be the Jacobian of $\tfrac{p_{\ry|\rvx}(y|G(\vz))}{p_{\ry}(y)}$ at $\vz_0$,
\begin{small}
\begin{align*}
\int_{\Omega_i} p_{\rvz}(\vz) h(\vz-\vz_0) \dif \vz &\approx \int_{\Omega_i} p_{\rvz}(\vz) \bm{g} (\vz-\vz_0) \dif \vz\\
&\stackrel{(b)}{\le} \textrm{const} \cdot \int_{\Omega_i} \bm{g} (\vz-\vz_0) \dif \vz \le \textrm{const} \cdot \int_{\Omega_i} \left|\bm{g} (\vz-\vz_0)\right| \dif \vz \\
&\stackrel{(c)}{\le} \textrm{const} \cdot \int_{A(\vz_0, Kr)} \left|\bm{g} (\vz-\vz_0)\right| \dif \vz \le \textrm{const} \cdot \int_{A(\vz_0, Kr)} \|\bm{g}\| \cdot \|\vz-\vz_0\| \dif \vz \\
&\stackrel{(d)}{\le} \textrm{const} \cdot \int_{A(\vz_0, Kr)} \|\vz-\vz_0\| \dif \vz
\end{align*}
\end{small}
where (b) follows since $p_{\rvz}(\vz)$ is bounded on a compact set, (c) follows since $\Omega_i \subset A(\vz_0, Kr)$, and (d) follows since the Jacobian has finite operator norm. The final bound is computable in closed form using polar coordinates~\citep{stromberg1981introduction},
\begin{align*}
\int_{\Omega_i} p_{\rvz}(\vz) h_i(\vz-\vz_0) \dif \vz 
&\le \textrm{const} \cdot \int_{A(\vz_0, Kr)} \|\vz-\vz_0\| \dif \vz \\
&= \textrm{const} \cdot S^{d-1} \int_{0}^{Kr} \rho \cdot \rho^{d-1} \dif\rho \\
&= \textrm{const} \cdot S^{d-1} \frac{(Kr)^{d+1}}{d+1} \\
&= \textrm{const} \cdot r^{d+1},
\end{align*}
where $S^{d-1}$ is the surface area of the unit sphere in $\mathbb{R}^{d}$. Substituting this into (~\ref{eq:identity_with_Taylor}), we have the following.
\begin{align*}
p_{\rvx}(G(\vz_y) \in B) &= \sum_{i=1}^t p_{\rvz_y}\left(\vz_y \in  \Omega_i\right) \\
&\le p_\rvx(B) \frac{p_{\ry|\rvx}(y|\vx_0)}{p_{\ry}(y)} + \sum_{i=1}^t \textrm{const} \cdot r^{d+1} \\
&\le p_\rvx(B) \frac{p_{\ry|\rvx}(y|\vx_0)}{p_{\ry}(y)} + \textrm{const} T \cdot r^{d+1}.
\end{align*}

Normalizing both sides by the volume of $B$, $\mu(B) = \textrm{const} \cdot r^{d}$, and taking $r \to 0$, the property of Lebesgue points (\ref{eq:Leb_eq1}) concludes that
\begin{align*}
&\frac{p_{\rvx}(G(\rvz_y) \in B)}{\mu(B)} = \frac{p_\rvx(B)}{\mu(B)} \frac{p_{\ry|\rvx}(y|\vx_0)}{p_{\ry}(y)} + \textrm{const}\cdot r \\
\implies &p_{\rvx}(G(\rvz_y) = \vx_0) \to \frac{p_{\rvx}(\vx_0) p_{\ry|\rvx}(y|\vx_0)}{p_{\ry}(y)} = p_{\rvx|\ry}(\vx_0|y) ~~ \textrm{as $r \to 0$}.
\end{align*}
Since the argument holds for arbitrary $\vx_0$, $G(\rvz_y) \eqd \rvx|\vy=y$.
\end{proof}

\setcounter{proposition}{0}

\subsection{\td{Full statements and proofs} for Section~\ref{sec:inrep_analysis}}
\label{app:proof_33}
We need one simple proposition before showing the theorem.

\begin{proposition}[restatement]
When the ideal unconditional generator $G$ and arbitrary modifier $M_y$ are fixed, the optimal discriminator for $y$ is 
\begin{align*}
    D_y^*(M_y(\vu)) = \frac{(1+\pi_y)p_{\text{data}}(\vx|y)}{(1+\pi_y) p_{\text{data}}(\vx|y) + (1-\pi_y) p_\text{gf}(\vx)}.
\end{align*}
\end{proposition}

\begin{proof}
Using $p_{G(M_y(\rvu))}(\vx) = \pi_y p_{\text{data}}(\vx|y) + (1-\pi_y) p_\text{gf}(\vx)$, we can rewrite $V_y^{PU}$.
\begin{small}
\begin{align*}
V_y^{PU} = {} & (1+\pi_y) \mathbb{E}_{\rvx \sim p_{\text{data}}(\vx|y)}[\log D_y(\rvx)] + \mathbb{E}_{\rvx \sim p_{G(M_y(\rvu))}(\vx)}[\log(1-D_y(\rvx))] -\pi_y\mathbb{E}_{\rvx \sim p_{\text{data}}(\vx|y)}[\log(1-D_y(\rvx))] \\
\begin{split}
    = {} & \int_{\vx} \Big[ (1+\pi_y) p_{\text{data}}(\vx|y) \log D_y(\vx) + \left(\pi_y p_{\text{data}}(\vx|y) + (1-\pi_y)p_\text{gf}(\vx) \right)   \log(1-D_y(\vx))  \\
    &- \pi_y p_{\text{data}}(\vx|y) \log(1-D_y(\vx)) \Big] \dif \vx
\end{split}
\\
= {} & \int_{\vx} (1+\pi_y) p_{\text{data}}(\vx|y) \log D_y(\vx) + (1-\pi_y) p_\text{gf}(\vx) \log(1-D_y(\vx)) \dif \vx.
\end{align*}
\end{small}
Since $\int a \log y + b \log (1-y)$ attains its maximum at $y^* = \frac{a}{a+b}$, we can conclude that
\begin{align*}
    D_y^*(M_y(\vu)) = \frac{(1+\pi_y)p_{\text{data}}(\vx|y)}{(1+\pi_y) p_{\text{data}}(\vx|y) + (1-\pi_y) p_\text{gf}(\vx)}.
\end{align*}
\end{proof}

\begin{theorem}[full statement]
When the ideal unconditional generator $G^*$ and discriminator $D^*$ are fixed, the global optimal modifier $M_y$ is attained if and only if $p_\text{gf}(\vx) = p_{g}(\vx) = p_{\text{data}}(\vx|y)$. Moreover, the loss achieved at $p_{g}(\vx) = p_{\text{data}}(\vx|y)$ is $(1+\pi_y) \log \tfrac{(1+\pi_y)}{2} + (1-\pi_y) \log \tfrac{(1-\pi_y)}{2}$.
\end{theorem}
\begin{proof}
Note that $((1+\pi_y) p_{\text{data}} + (1-\pi)p_\text{gf})/2$ is a valid probability distribution. Since $D_y^*$ in Proposition 2 is fixed,
\begin{align*}
    \begin{split}
         V_y^{PU} = {} &\int_{\vx} \Big[ 
    (1+\pi_y) p_{\text{data}}(\vx | \ry=y) \log \frac{(1+\pi_y) p_{\text{data}}(\vx | \ry=y)}{(1+\pi_y) p_{\text{data}}(\vx | \ry=y) 
    + (1-\pi_y)p_\text{gf}} 
    \\ &+ (1-\pi_y)p_\text{gf} \log \frac{(1-\pi_y)p_\text{gf}}{(1+\pi_y) p_{\text{data}}(\vx | \ry=y) + (1-\pi_y)p_\text{gf}} \Big] \dif \vx 
    \end{split}\\
    \begin{split}
        ={} & (1+\pi_y) \int_{\vx} p_{\text{data}}(\vx | \ry=y) \log \left( \frac{p_{\text{data}}(\vx | \ry=y)}{((1+\pi_y) p_{\text{data}}(\vx | \ry=y) + (1-\pi_y)p_\text{gf})/2} \cdot \frac{1+\pi_y}{2} \right) \dif \vx \\
    & + (1-\pi_y) \int_{\vx} p_\text{gf} \log \left( \frac{p_\text{gf}}{((1+\pi_y) p_{\text{data}}(\vx | \ry=y) + (1-\pi_y)p_\text{gf})/2} \cdot \frac{1-\pi_y}{2} \right) \dif \vx 
    \end{split}\\
    \begin{split}
        = {} &(1+\pi_y) \mathrm{KL}\left(p_{\text{data}}(\vx | \ry=y) \middle\| \frac{p_{\text{data}}(\vx | \ry=y)}{((1+\pi_y) p_{\text{data}}(\vx | \ry=y) + (1-\pi_y)p_\text{gf})/2}\right) + (1+\pi_y) \log \frac{1 + \pi_y}{2} \\
        & + (1-\pi_y) \mathrm{KL}\left(p_\text{gf} \middle\| \frac{p_{\text{data}}(\vx | \ry=y)}{((1+\pi_y) p_{\text{data}}(\vx | \ry=y) + (1-\pi_y)p_\text{gf})/2}\right) + (1-\pi_y) \log \frac{1 -\pi_y}{2},
    \end{split}
\end{align*}
where $\mathrm{KL}(\cdot\|\cdot)$ is the Kullback–Leibler divergence. Optimizing $M_y$ yields the global optimum $p_\text{gf}(\vx) = p_{\text{data}}(\vx | y)$ and therefore, $p_\text{gf}(\vx) = p_{\text{data}}(\vx | y) = p_{G(M_y(\rvu))}(\vx)$.
\end{proof}

\section{Details of Experiment Settings and Implementation}
\label{sec:app_experiment}

\subsection{Datasets} 

\textbf{\dm{}}~\citep{lecun-mnisthandwrittendigit-2010} contains ten classes of   $28\times 28$ black-and-white images, with \num{50000} training images and \num{10000}  testing images.
\textbf{\dc{}}~\citep{Krizhevsky09learningmultiple} dataset is a widely used benchmark dataset in image synthesis. 
The dataset contains ten classes of three-channel (color) $32\times32$ pixel images, with \num{50000} training and \num{10000} testing images.
\textbf{CIFAR100}~\citep{Krizhevsky09learningmultiple} contains 100 classes of three-channel (color) $32\times32$ pixel images (similar to \dc{} images).
Each class contains 600 images: 500 images for training and 100 images for testing.
\textbf{CelebA}~\citep{liu2015faceattributes} is a large-scale dataset of celebrity face that has more than \num{200000} images with 40 attributes for each image.
\textbf{Flickr-Faces-HQ (FFHQ)}~\citep{karras2019style} is the dataset of human faces, which consists of \num{70000} high-quality images at $1024\times1024$ resolution,  crawled from Flickr.

\subsection{Evaluation metrics}

\textbf{Frechet Inception Distance (\fid{})} is a widely-used metric to measure the quality and diversity of learned distributions in the \gan{} literature. 
\fid{} is Wasserstein-2 (Fr\'{e}chet) distance between Gaussians fitted to the data embedded into the feature space of the Inception-v3 model (pool3 activation layer).
Lower \fid{} indicates better \gan{}'s performance.
As \fid{} relies on a feature space from a classifier trained on \dimg{}, \fid{} is not suitable for non-\dimg{}-like images (e.g., \dm{} or \df{}). 
For conditional \gan{}s, we also calculate the \fid{} for each particular class data and get the average score as \cfid{} score~\citep{miyato2018cgans}.
This score measures class-conditioning performances.

\textbf{Classification Accuracy Score (CAS)}~\citep{ravuri2019classification} measures the accuracy of a classifier trained with the generated conditional data.
More specifically, we train a classifier using labeled data generated by conditional generators and measure the classification accuracy of the trained classifier on the real dataset as \cas{}.
Intuitively, if the generated distribution matches the real distribution, \cas{} should be close to the accuracy of the classifier trained on real samples.

\textbf{Recall}~\citep{kynkaanniemi2019improved} measures how well the learned distribution covers the true (or reference) distribution.
The value of the recall score is between $0$ and $1$.
Higher recall scores indicate better coverage.

\subsection{Architectures, hyperparameters, and implementation}

Our goal is to compare the effectiveness of \rg{} with other methods, rather than to produce the model with the state-of-the-art performance on data synthesis.
The latter part usually requires much more computations and tricks.
Therefore, we conduct fair and reliable assessments of conditioning methods on widely used \gan{}s.
We apply the same configuration for all models.

We use different configurations for different datasets.
For the \textbf{\dg{}},
we use a simple network architecture with three blocks. 
Each block consists of a fully connected layer followed by a batch-normalization and a ReLU activation function.
We use non-saturating \gan{} loss~\citep{goodfellow2014generative}.
We train each model with \num{10000} steps (batch size of $64$), using Adam Optimizer ($\beta_1 = 0.5, \beta_2= 0.999$) with learning rate $2\cdot10^{-4}$ for all networks.
For \textbf{\dm{}},
We adopt WGAN-GP~\citep{martin2017wasserstein} with DCGAN architectures for \dm{}~\citep{mehralian2018rdcgan}.
The loss is Wasserstein loss with gradient penalty regularization.
We set the number of critic steps to be five. 
We train each model \num{50000} steps (batch size of 64) using Adam Optimizer ($\beta_1 = 0.5, \beta_2= 0.999$) with learning rates $2\cdot10^{-4}, 2\cdot10^{-5}$ for modifier and discriminator networks, respectively.
On \textbf{\dc{} and \dhc{}},
among three widely used architectures in the \gan{} literature (DCGAN, ResGAN, and BigGAN), we adopt ResGAN for all models  
given that~\citep{kurach2019large} observes the comparable performance between DCGAN and ResGAN on various settings of \gan{}s.
Also, we apply best practices and recommended configurations for each model~\citep{kurach2019large,lucic2019high}: spectral normalization (an essential element in modern \gan{} training), Hinge loss~\citep{lim2017geometric}.
For \textit{\rg{}}, 
the modifier network adopts the architecture of i-ResNet~\citep{behrmann2019invertible}, with five layers for FFHQ data and three layers for CIFAR datasets.

\end{document}